\documentclass{article}

    \usepackage{defs}

\usepackage[utf8]{inputenc} % allow utf-8 input
\usepackage[T1]{fontenc}    % use 8-bit T1 fonts
\usepackage{hyperref}       % hyperlinks
\usepackage{url}            % simple URL typesetting
\usepackage{booktabs}       % professional-quality tables
\usepackage{amsfonts}       % blackboard math symbols
\usepackage{nicefrac}       % compact symbols for 1/2, etc.
\usepackage{microtype}      % microtypography
\usepackage{xcolor}         % colors

\usepackage{preamble} 

\usepackage{microtype}
\usepackage{graphicx}
\usepackage{subfigure}
\usepackage{booktabs} 
\usepackage{xcolor}
\usepackage{amsthm}
\usepackage{thmtools, thm-restate}
\usepackage{multibib}
\newcites{Appendix}{References}

\newenvironment{proofS}{\proof}{\endproof}

\newcommand{\sti}[1]{}

\usepackage[compact]{titlesec}
\titlespacing{\section}{0pt}{*2.0}{*2.0}
\titlespacing{\subsection}{0pt}{*0.5}{*0.5}
\titlespacing{\paragraph}{0pt}{*0.0}{*1}

\newcommand{\ignore}[1]{}

\usepackage{authblk}

\usepackage{hyperref}

\begin{document}

% \author{Author XXX\textsuperscript{a}\thanks{\noindent\textsuperscript{a}Address of the author XXX \\
% Email...}\,
% Author YYY \textsuperscript{b}{\footnote{\textsuperscript{b}
% Address of the author YYY}} }
% \date{}

% \author[1, 2]{\small Lee Cohen}
% \author[2,3]{\small Yishay Mansour}
% \author[1,4]{\small Ulrike Schmidt-Kraepelin}
% \author{Lee Cohen
% \thanks{%\reviewStepan{
% These authors contributed equally to this work.}$^{,*}$
% \quad  Yishay Mansour$^{,\dagger}$ \quad  Ulrike Schmidt-Kraepelin$^{*,3}$ \\
% $^1$ Blavatnik School of Computer Science, Tel Aviv University\;
% $^2$ Google Research\;
% $^3$ Technische Universität Berlin
% }\date{}

\author{Lee Cohen\thanks{
These authors contributed equally to this work.

\;$^1$ Blavatnik School of Computer Science, Tel Aviv University.

\;$^2$ Research Group Efficient Algorithms, Technische Universität Berlin.

\;$^3$ Google Research.}$^{*,1}$ \quad  Ulrike Schmidt-Kraepelin$^{*,2}$ \quad  Yishay Mansour$^{1,3}$  \\
% $^1$ Blavatnik School of Computer Science, Tel Aviv University\quad
% $^2$ Google Research
% $^3$ Technische Universität Berlin\quad
}\date{}

% \affil[1]{\footnotesize These authors contributed equally to this work.}
% \affil[2]{\footnotesize Blavatnik School of Computer Science, Tel Aviv University}
% \affil[3]{\footnotesize Google Research}
% \affil[4]{\footnotesize Technische Universität Berlin}
% \affil[5]{\footnotesize Einstein and Spitzer Fellow}

\title{Dueling Bandits with Team Comparisons}

\maketitle

\begin{abstract}%
  We introduce the \emph{dueling teams problem}, a new online-learning setting in which the learner observes noisy comparisons of disjoint pairs of $k$-sized \emph{teams} from a universe of $n$ players. The goal of the learner is to minimize the number of duels required to identify, with high probability, a \textit{Condorcet winning team}, i.e., a team which wins against any other disjoint team (with probability at least $1/2$).
Noisy comparisons are linked to a total order on the teams. 
We formalize our model by building upon the dueling bandits setting \citep{Yue2012} and provide several algorithms, both for stochastic and deterministic settings. For the stochastic setting, we provide a reduction to the classical dueling bandits setting, yielding an algorithm that identifies a Condorcet winning team within $\mathcal{O}((n + k \log (k)) \frac{\max(\log\log n, \log k)}{\Delta^2})$ duels, where $\Delta$ is a gap parameter. For deterministic feedback, we additionally present a gap-independent algorithm that identifies a Condorcet winning team within $\mathcal{O}(nk\log(k)+k^5)$ duels.

\end{abstract}

\section{Introduction}
Multi-arm bandits (MAB) is a classical model of decision making under uncertainty. 
In spite of the simplicity of the model, it already incorporates the essential tradeoff between exploration and exploitation. In MAB, the learner performs actions and can only observe rewards of the actions performed. One of the main tasks in MAB is 
\emph{best arm identification}, where the goal is to identify a near-optimal action while minimizing the number of actions executed. The MAB model has numerous practical applications, including online advertising, recommendation systems, clinical trials, and more. (See \cite{Slivkins-book,Lattimore-Szepesvari-Book} for more background).

One weakness of the MAB model is the assumption that  real-valued rewards are always available. In many applications, it is more natural to compare two actions and observe which one of them is better rather than give every  single action a numerical reward. For example, recommendation systems often suggest two items and obtain only their relative preference as feedback (e.g., by a click on one of them). This leads very naturally to the well-known model of dueling bandits \citep{Yue2012}, where the learner selects a pair of actions each time and observes the binary ``winner'' of a duel between the two. (See \cite{Busa2018} for a survey on extensions of the dueling bandit model).

In this work we are interested in the case that the learner has to select two disjoint \textit{teams} for a duel, which are $k$-sized subsets of the actions (which we call players).
This appears naturally in sports or online games, where the goal is to pick one of the best teams from a set of players by observing the outcomes of matches (say, to be a school representative team, or to sponsor for tournaments). Examples include doubles tennis, basketball, and the online game League of Legends, where each match requires two disjoint teams of players to compete.
Similar phenomena appear in working environments, where different R\&D teams compete on implementing a project. Another example could be 
online advertisements where multiple products are bundled to a display ad and a customer can click on any of two presented bundles, 
e.g.,
some online games 
offer in-app bundle purchases, 
and the information regarding 
sales of different bundles can be used to improve the bundles' composition. 

Our basic model is the following. We have a universe of $n$ players, and at each iteration the learner selects two \textbf{disjoint} teams for a duel and observes the winner. For any two different teams, there exists an unknown stationary probability that determines the winner of a duel between them. 
The requirement that teams need to be disjoint is in accordance with the situation in 
games, where a single person cannot play for both teams.
The goal of the learner is to minimize the number of duels required to identify, with high probability, a \textit{Condorcet winning team}, i.e., a team which 
wins against any other disjoint team (with a probability of at least $1/2$).
We do assume 
these probabilities are linked to a strict total order on 
all teams, which implies the existence of a Condorcet winning team, yet it is typically not unique.
We make two minimal and natural assumptions on this total order on 
teams, namely, that it is \emph{consistent} to some total order among the players, and that the team probabilistic comparisons hold \emph{Strong Stochastic Transitivity}, a common assumption in dueling bandit settings.

Clearly, given any 
total order among the players, the best team is the one containing the top $k$ players, which is in particular one of the Condorcet winning teams. 
However, not all relations between players are deducible for the learner.
In particular, even achieving  accurate estimations of the latent winning probabilities between all disjoint teams might not suffice to separate the top $k$ players from the rest.
Consider for example an instance with four players $1\succ2\succ3\succ4$ where $k=2$ and the total order among the teams is lexicographical, i.e., $12\succ13\succ14\succ23\succ 24\succ34$. Then, there exist three feasible duels, each of which is won by the team containing player $1$ with probability greater than $1/2$. If all three duels are won with equal probability by the team containing $1$, the learner has no chance of detecting the team $12$ as the top $k$ team. However, any of the teams $12,13$ and $14$ is a Condorcet winning team.

Our main target is to present algorithms for which the number of duels is bounded by a polynomial in the number of players $n$ and team size $k$, although the number of teams is exponential in $k$, i.e., $\Omega((\frac{n}{k})^k)$ and the number of valid duels is $\Omega(2^{k}(\frac{n}{2k})^{2k})$. Even if one were to accept an exponential number of arms, a direct reduction to the standard dueling bandits setting would not be feasible as not all pairs of teams are comparable in our model. In particular, duels of the form $(S \cup \{a\}, S \cup \{b\})$, which would yield a signal regarding the relation between players $a$ and $b$, are forbidden.
The inherent difficulty of our endeavor comes from 
two limitations: 
(1) Not all the relations between two single players are deducible, (see example above), and (2) even for pairs of players with deducible relation, having $\Omega(2^k(\frac{n}{2k})^{2k})$ valid duels and the same amount of (latent) winning probabilities makes the task of deducing their relations hard.

We start by giving a full characterization of the \emph{deducible} pairwise relations between players, namely relations that can be detected by a learner which is allowed to perform an unlimited amount of duels. 
Our characterization implies that every deducible single player relation has one of two types of \textit{witnesses}, which are constant-size sets of duels that prove their relation. % We also show that, once we find a witness for some pairs of players and how to that a witness for one pair of players can often be transformed to a witness for other pairs of players.
We also show that, once we find a witness for one pair of players, it can often be transformed to a witness for other pairs of players.

Building upon this characterization, we introduce a parameter $\Delta_{a,b}$ which captures the distinguishability of any two players $a$ and $b$ and takes a value of $0$ whenever the pair is not deducible.
Assuming $\Delta := \Delta_{k,k+1}>0$, where $k$ and $k+1$ are 
$k^{\text{th}}$ and $(k+1)^{\text{th}}$ best players, we give
a reduction to the classic dueling bandits problem. 
Combining this reduction with a high-probability top-$k$ identification algorithm for the dueling bandits setting (e.g., \cite{mohajer17,ren2020}) yields a similar sample complexity upper bound, e.g., this yields a high-probability top-$k$ identification algorithm for dueling teams with $\mathcal{O}(\Delta^{-2}(n + k \log (k)) \max(\log\log n, \log k))$ duels.

Interestingly, it turns out that the deterministic case, i.e., when winning probabilities are in $\{0,1\}$, constitutes a challenging special case of our problem where $\Delta$ can be particularly small, or even $0$. To overcome this issue we design delicate algorithms which are independent of $\Delta$. On a high level, a preprocessing procedure first excludes as many bad players as possible. To do so, it runs a method for identifying pairwise relations between players which performs only a small number of duels, but has little control over the pair for which the relation is uncovered. For general 
total orders this implies an algorithm requiring $\mathcal{O}(nk\log(k) + 2^{\mathcal{O}(k)})$ duels. For the natural case of \textit{additive linear} orders, we present a more elaborated approach for detecting a Condorcet winning team within the reduced instance, resulting in an algorithm that performs $\mathcal{O}(nk\log(k) + k^5)$ duels. 

We introduce our problem in Section \ref{sec:model}, give a characterization of deducible relations in Section \ref{sec:witnesses2}, discuss the stochastic setting in Section \ref{sec:stochastic}, and the deterministic setting in Section \ref{sec:deterministic}. For brevity, algorithms and (full) proofs are relegated to Section \ref{sec:witnessesFull}, \ref{sec:stochasticFull}, and \ref{sec:deterministicFull} of the appendix. Section \ref{sec:extensions} contains a discussion and Section \ref{sec:additiveLinear} a characterization of additive linear total orders.

\medskip
\subsection{Related Work}
\medskip 

\noindent{\bf MAB best arm or subset identification:} single arm identification was initiated in \cite{Even-DarMM06} and later studied in many works including \cite{BubeckMS11,KaufmannCG16,ChenLQ17}. This setting was extended by \cite{Kalyanakrishnan10}  for multiple arms identification (i.e., top $k$ arms), using a single arm samples. Other works that address the objective of top$-k$ identification include \cite{Chen14,Zhou14,Bubeck13}.

\noindent{\bf Dueling bandits}
The work of \cite{Yue2012} lay down the framework of non-parametric bandit feedback under total order among arms, strong stochastic transitivity, and stochastic triangle inequality assumptions and were followed by many subsequent works (For more, see a survey, \cite{Busa2018}.) 
In particular, some subsequent works target the task of identifying the top $k$ players in this setting \cite{mohajer17,ren2020}.

\noindent{\bf Dueling bandits with sets of actions} 
One line of dueling bandits extension consider the case where the learner selects a subset of actions and observes the outcomes of all duels between all pairs of actions in the subset \citep{Brost+16,Sui+17}, or the winner of the subset \cite{SG18,Ren+18}. 
As a consequence, these settings give the learner strictly more information than the dueling bandits setting. In contrast, feedback in our setting reveals less information.

\noindent{\bf MAB with multiple actions selection
:} There are works in which the learner selects a (sometimes fixed-sized) subset of actions at each iteration, and observes either all of the individual selected arms rewards (semi-bandit feedback) or an aggregated form of the rewards (full-bandit feedback), and the task is to detect to best arm or the top $k$. These include  \emph{combinatorial bandits}
\cite{Cesa-BianchiL12}, \emph{top-k}
\cite{RejwanM20}, \emph{linear bandit and routing}
\cite{AwerbuchK08}, and more. 
The main difference between combinatorial bandits and our setting is the feedback.

\noindent{\bf Comparison models: } 
Noisy pairwise comparison models, especially for sorting and ranking, have a long history which dates backs to the 1950's (For more, see a survey, \cite{DPelc02}.).
Specifically, the mathematical problem Counterfeit coin 
was introduced in the form of a puzzle 
\citep{Grossman45}:
given a pile of $12$ coins, determine which 
coins has a different weight (and therefore counterfeit) using balance scales while minimizing the number of measurements. The problem was followed by numerous generalizations (see \cite{Guy95}). 
While this problem is restricted to coins with two different weights, our setting can be seen as a variant with multiple weights. 

\section{The Dueling Teams Problem}

We formalize our problem as follows. Let $n,k \in \mathbb{N}$ with $1 \leq k \leq \frac{n}{2}$. We denote the set of players by $[n]:=\{1, \dots, n\}$ and call any set of $k$ distinct players a \emph{team}. Moreover, we assume the existence of an underlying strict total order among all teams, and denote it by $\succ$. We also refer to $\succ$ as the ground truth order. In particular, for any two teams $A$ and $B$ either $A \succ B$ holds, in which case we say that $A$ is \emph{better} than  $B$, or vice versa, and this relation is transitive. Additionally, we require the total order among the teams to be consistent with a total order among players and formalize this in the \emph{consistency} assumption at the end of this section. 

In each round, the learner selects an ordered pair of two disjoint teams, $A$ and $B$ to perform a \emph{duel}, and receives a noisy binary feedback about which team is better.
Note that in contrast to the usual dueling bandits setting, our setting does not allow duels of the form $(A,A)$, as selecting teams with mutual players for a duel is not an option. We denote the \emph{observable} part of $\succ$ by $\succ_{obs}$, i.e., $A \succ_{obs} B$ iff $A$ and $B$ are disjoint teams and $A \succ B$. 
Note, $\obs$ is not transitive, thus not even a partial order.

We write $A>B$ if team $A$ is the random winner of duel $(A,B)$. The probability $\Pr[A>B]$ is stationary and denoted by $P_{A,B}=\Pr[A>B]$. 
In each duel of team $A$ against team $B$ the outcome $A>B$ is sampled independently from a Bernoulli 
distribution with parameter $P_{A,B} = 1 - P_{B,A}$. 
We assume that the probabilistic comparisons are linked to the total order among the teams, i.e., $A \succ B$ implies $P_{A,B}> 1/2$, and that $P_{A,B}$ exists for every pair of teams (not only disjoint ones).

In the deterministic setting, it holds that $P_{A,B}\in\{0,1\}$ for any teams $A\ne B$. In other words, $A\obs B$ iff the outcome of each duel $(A,B)$ is $A>B$, and for two disjoint teams $A$ and $B$ the learner can observe whether $A \succ_{obs} B$ or $B \succ_{obs} A$   
by performing a single duel.

A team $A$ is a \emph{Condorcet winning team}\footnote{The name is motivated by the fact that such a team is a weak Condorcet winner for the relation $\succ_{obs}$.} if $A \succ_{obs} B$ for all teams $B$ such that $A \cap B = \emptyset$. From our assumption on $\succ$
, there always exists a Condorcet winning team, but it is not necessarily unique. 
The learner's goal is to minimize the number of duels required to  identify, with high probability in the stochastic setting and with probability $1$ in the deterministic case, a Condorcet winning team.

In the following we formalize two more assumptions we impose on our model, the former affects the linking of the probabilities to the strict total order $\succ$, the latter restricts the total order $\succ$ itself.

\noindent\textbf{Strong stochastic transitivity (SST):}
Similarly to the dueling bandits settings in \cite{Yue2012}, we assume \textit{strong stochastic transitivity}. Namely, for every triplet of different teams $A\succ B\succ C$ it holds that 
$
P_{A,C}\geq \max\{P_{A,B},P_{B,C}\}.
$

\noindent\textbf{Consistency:}
We assume that the total order $\succ$ is consistent to a total order among single players. More precisely, we say that $\succ$ satisfies \emph{consistency} if for every two players $a,b\in [n]$ either of the following holds true: 
\begin{enumerate}[label=(\roman*)]
    \item $S\cup\{a\}\succ S \cup \{b\}$ for all $S \subseteq [n] \setminus\{a,b\},  |S|=k-1$.%, or
    \item     $S\cup\{b\}\succ S \cup \{a\}$ for all $S \subseteq [n] \setminus\{a,b\},  |S|=k-1$.%, |S|=k-1.$
\end{enumerate}
The consistency assumption lets us derive a relation among the single players, by defining $a \succ b$ iff $S \cup \{a\} \succ S \cup \{b\}$ holds for some $S$. By team relation transitivity, $\succ$ implies a total order on $[n]$. Whenever we write $a \succ b$ for some players $a,b \in [n]$ this is short-hand notation for $S \cup \{a\} \succ S \cup \{b\}$ for all subsets $S \subseteq [n] \setminus \{a,b\}$ of size $k-1$. 
For notational convenience, we assume without loss of generality that $1\succ 2\succ \dots \succ n$ and write $A^*_m$ for the set of players containing the top $m$ players, i.e., $A^*_m=[m]$. In particular, the consistency assumption yields that $A^*_k$ is a Condorcet winning team.

Though the ground truth ranking induces a total order among the players, the learner might not be able to deduce the entire order. In the following we give a characterization of the \emph{deducible} part of the ground truth order $\succ$.

\label{sec:model}

\section{Witnesses: A Characterization of Deducible 
Relations}\label{sec:witnesses2}
In this section we provide a high level description of the complete characterization of all the pairwise relations between single players that can be deduced via team duels. Though single players  cannot be observed via team duels directly, we show a sufficient and necessary condition for deducible relations in the form of a constant number of winning probabilities of observable (feasible) duels. We refer to a set of players participating in such duels as \emph{witnesses}. For completeness, we point out that a similar characterization can be done for any same-sized subsets of size less than $k$.

We denote by $\mathbb{P}_{obs}$ the set of all tuples $(P',\succ')$, where each $P'$ is a team winning probability matrix that satisfy SST w.r.t. $\succ'$, which is a consistent strict total order on teams, and both $P'$ and $\succ'$ are compatible with the winning probabilities of observable duels and each other, i.e., $\{P'_{A,B} = P_{A,B} \mid A \text{ and } B \text{ are disjoint teams}\}$ and $P'_{A,B}=1-P'_{B,A}>1/2$ iff $A\succ'B$.
We remark that it follows directly from the definition of $\mathbb{P}_{obs}$ that $(P,\succ) \in \mathbb{P}_{obs}$, where $P$ is the ground truth winning probability matrix and $\succ$ the ground truth total order.

We denote by $\mathcal{C}_{obs}$ the set of strict total orders $\succ'$ for which there exists a tuple $(P',\succ')\in \mathbb{P}_{obs}$.

More precisely, $\succ' \in \mathcal{C}_{obs}$ if $\succ'$ is a total order on all teams that satisfies consistency and there exist probabilities $P'_{A,B}$ for all pair of teams $(A,B)$ such that $A \succ' B$ iff $P'_{A,B}>1/2$, and $P'_{A,B} = P_{A,B}$ for all disjoint teams $A$ and $B$. Lastly, we define $ A \succ^* B$ if and only if $A \succ' B$ for all $\succ' \in \mathcal{C}_{obs}$, where $A$ and $B$ are not necessarily disjoint. We refer to $\succ^*$ as the \emph{deducible} relation. For single player relations, we define  $ a \succ^* b$ if and only if there exists $S\subseteq [n]\setminus \{a,b\}$ such that $S\cup \{a\}\succ' S \cup \{b\}$ for all $\succ' \in \mathcal{C}_{obs}$. We stress that we only use $\mathbb{P}_{obs}$ and  $\mathcal{C}_{obs}$ for analysis and never actually compute them.

Next, we define two sets of \emph{potential witnesses} that have a simple structure and, in some cases,  
allow us to deduce single players relation: (1) A \emph{potential subsets witnesses} set, denoted by $\mathcal{S}_{a,b}$, that contains all pairs $(S,S')$ such that $S$ and $S'$ are disjoint subsets of $[n] \setminus \{a,b\}$ and both are of size $k-1$, and (2) A \emph{potential subset-team witnesses} set, denoted by $\mathcal{T}_{a,b}$, that contains all pairs $(S,T)$ where $S$ and $T$ are disjoint subsets of $[n] \setminus \{a,b\}$, such that $S$ is of size $k-1$ and $T$ is of size $k$ (and is therefore a team).
Below, we define under which conditions a potential witnesses is a \emph{witness}.

\begin{definition}
An element $(S,S')\in \mathcal{S}_{a,b}$ is a \emph{subsets witness} for $a \succ b$ if
$P_{S\cup \{a\},S' \cup \{b\}}>P_{S \cup \{b\},S' \cup \{a \}}$. An element $(S,T)\in \mathcal{T}_{a,b}$ is a \emph{subset-team witness} for $a \succ b$ if 
$P_{S \cup \{a\},T} > P_{S \cup \{b\},T}$.
\end{definition}\vspace{-.1cm}
We capture the set of the elements of $\mathcal{S}_{a,b}$ that are subsets witnesses for $a \succ b$ by $\mathcal{S}^*_{a,b}$ and analogously, $\mathcal{T}^*_{a,b} = \{(S,T) \in \mathcal{T}_{a,b} \mid (S,T) \text{ is a subset-team witness for } a \succ b \}$. It might be the case that $\mathcal{S}^*_{a,b} \cup \mathcal{T}_{a,b}^*$ is empty, in particular this holds when $b \succ a$. 
It is also possible that both $\mathcal{S}^*_{a,b} \cup \mathcal{T}_{a,b}^*$ and $\mathcal{S}^*_{b,a} \cup \mathcal{T}_{b,a}^*$ are empty, 
in which case we will show that the relation between players in $a$ and $b$ cannot be deduced.
The following theorem implies that the other direction is also true.

\begin{restatable}{theorem}{thmcharacterization}\label{thm:provableSingleRelationIffWitnessExists}
Let $a,b \in [n]$. Then, $a \succ^* b$ if and only if $\mathcal{S}^*_{a,b} \cup \mathcal{T}^*_{a,b} \neq \emptyset$. 
\end{restatable}
\begin{proofS}
Assume that $\mathcal{S}^*_{a,b} \cup \mathcal{T}^*_{a,b} \neq \emptyset$. We show that $a \succ^* b$ by using SST, the fact that $\succ$ is a consistent strict total order, and an exhaustive case analysis. For the sake of illustration we present only one case here, namely, that $(S,S') \in \mathcal{S}_{a,b}^*$ and that both $(1)\;S \cup \{a\} \succ S' \cup \{b\}$ and $(2)\;S \cup \{b\} \succ S' \cup \{a\}$ hold. Assume for contradiction that $a\succ^* b$ does not hold. It thus follows that there exists an order, $\succ' \in \mathcal{C}_{obs}$ for which $b \succ' a$ holds. Let $P'_{A,B}$ be the corresponding winning probabilities. Then, using consistency of $\succ'$ and $(1)$ respectively, we get $S \cup \{b\} \succ' S \cup \{a\} \succ' S' \cup \{b\}$ and from SST $P'_{S \cup \{b\},S' \cup \{b\}} \geq P'_{S \cup \{a\},S' \cup \{b\}}>1/2$. In addition, applying consistency again, it follows that $S \cup \{b\} \succ S' \cup \{b\} \succ S' \cup \{a\}$. Applying SST once more we get $P'_{S \cup \{b\},S' \cup \{a\}} \geq P'_{S \cup \{b\},S' \cup \{b\}} \geq P'_{S \cup \{a\},S' \cup \{b\}}$, a contradiction to $(S,S') \in \mathcal{S}_{a,b}^*$ (since this implies $P_{S \cup \{a\},S' \cup \{b\}} > P_{S \cup \{b\},S' \cup \{a\}}$).

For the other direction we start by defining $\mathcal{D}_a$ as the set of observable duels $(A,B)$ such that $a \in A$. Moreover, we define a permutation $\pi$ on the set of teams, which simply exchanges the players $a$ and $b$ when present. We then show that $a \succ b$ implies $P_{A,B} \geq P_{\pi(A),\pi(B)}$ for all $(A,B) \in \mathcal{D}_a$. Moreover, we show that $a \succ^* b$ implies that there exists $(A,B) \in \mathcal{D}_a$ with $P_{A,B} > P_{\pi(A),\pi(B)}$ as follows. Assume not. Then we show that the relation $\succ'$ defined by $A \succ' B$ iff $\pi(A) \succ' \pi(B)$ is included in $C_{obs}$. However, $a \succ^* b$ implies that for any $S \subseteq [n] \setminus \{a,b\}$ of size $k-1$ it holds that $S \cup \{a\} \succ^* S \cup \{b\}$ which implies $(i) \; S \cup \{a\} \succ S \cup \{b\}$ as well as $ (ii) \;S \cup \{a\} \succ' S \cup \{b\}$. Applying the definitions of $\succ'$ and $\pi$, statement $(ii)$ implies $S \cup \{b\}=\pi(S \cup \{a\}) \succ \pi(S \cup \{b\})=S \cup \{a\}$ and hence yields a contradiction to $(i)$. 
Finally, take some $(A,B) \in \mathcal{D}_a$ with $P_{A,B}>P_{\pi(A),\pi(B)}$. If $b \in B$, then $(A \setminus \{a\},B \setminus \{b\}) \in \mathcal{S}^*_{a,b}$, otherwise $(A \setminus \{a\},B) \in \mathcal{T}^*_{a,b}$. 
\end{proofS} \vspace{-.3cm}
For the sake of brevity, we introduce the set $\mathcal{X}_{a,b}$ which combines the pairs from $\mathcal{S}_{a,b}$ and $\mathcal{T}_{a,b}$ into a set of triples. Formally, $\mathcal{X}_{a,b} = \{(S,S',T) \mid (S,S') \in \mathcal{S}_{a,b}, (S,T) \in \mathcal{T}_{a,b}\}.$ We say that $(S,S',T)$ is a witness for $a \succ b$ if $(S,S') \in \mathcal{S}^*_{a,b}$ or $(S,T) \in \mathcal{T}^*_{a,b}$, and denote $(S,S',T)\in \mathcal{X}^*_{a,b}$.

\section{Stochastic Setting}\label{sec:stochastic}

In this section we focus on algorithms identifying, with high probability, the top-$k$ team, which is in particular a Condorcet winning team. The main idea is to reduce the dueling teams setting to the classic dueling bandits setting, by  
which we refer to \cite{Yue2012}. To this end we will
introduce our \emph{gap parameter}, $\Delta$, which intuitively captures how easy it is to prove the relationship between the top-$k$ and the top-$(k+1)$ player. We start by defining, for any element of $\mathcal{X}_{a,b}$, a random variable $X_{a,b}(S,S',T)$ 
that combines the outcomes of the four duels which help 
determines whether $(S,S',T)$ is a witness for $a \succ^* b$. Formally, \begin{align*}X_{a,b}(S,S',T) &= \big(\mathbbm{1}[S \cup \{a\} > S' \cup \{b\}] - \mathbbm{1}[S \cup \{b\} > S' \cup \{a\}] \\
& \quad + \mathbbm{1}[S \cup \{a\} > T] - \mathbbm{1}[S \cup \{b\}>T] \big)/4. \end{align*}
 
Observe that 
$X_{a,b}(S,S',T)$ can take values from $\{-1/2,-1/4,0,1/4,1/2\}$, thus  $\E[X_{a,b}(S,S',T)] \in [-1/2,1/2]$. Moreover, we have the following properties:
\begin{enumerate}
    \item For every $(S,S',T)\in \mathcal{X}_{a,b}$ we have 
    \[
    \E[X_{a,b}(S,S',T)]>0 \iff (S,S',T)\in \mathcal{X}^*_{a,b}.
    \]
    \item If $\E[X_{a,b}(S,S',T)]=0$ for every $(S,S',T)\in \mathcal{X}_{a,b}$, then Theorem \ref{thm:provableSingleRelationIffWitnessExists} implies that the pairwise relation between players $a,b$ cannot be deduced.
\end{enumerate}
Building upon the random variables $X_{a,b}(S,S',T)$, which are defined for a fix pair of players, $a,b$, and for each element in $\mathcal{X}_{a,b}$,
we define a single random variable $X_{a,b}$ by picking a random triplet $(S,S',T) \in \mathcal{X}_{a,b}$ and returning a realization of $X_{a,b}(S,S',T)$. For convenience, whenever we write $\E[X_{a,b}]$ we mean $\E_{(S,S',T)\sim \mathcal{X}_{a,b}}[X_{a,b}]$. Using 
the probabilistic method, we obtain the following theorem, which then brings us to the definition of a gap parameter for our problem.

\begin{restatable}{theorem}{thmTheOneTrueLemmaToRuleAllWitnesses}\label{thm:theOneTrueLemmaToRuleAllWitnesses}
For every two players $a,b\in[n]$ it holds that $a\succ^* b$ if and only if $\;\E[X_{a,b}]>0$.
\end{restatable}
\noindent\textbf{Gap parameter: }
We define our gap parameter by $\Delta:=\E[X_{k,k+1}]$. 
In the following we show that our gap parameter does not just
help us to distinguish between the top $k$ and $k+1$ players, but also allows us to distinguish other players in $A^*_k$ and players from $[n] \setminus A^*_k$. To this end, we show in Lemma \ref{lemma:SSTforF} that the expectations $\E[X_{a,b}]$ satisfy strong stochastic transitivity w.r.t. the ground truth total order on players.
We note that most elements $(S,S',T) \in \mathcal{X}_{a,b}$ hold $\E[X_{a,c}(\pi(S),\pi(S'),\pi(T))] \geq \E[X_{a,b}(S,S',T)]$ (and analogously for $X_{b,c}$), where $\pi$ is a permutation exchanging players $b$ and $c$, but, surprisingly, this is not true in general. 
By carefully constructing a charging scheme, we manage to
show that this holds in expectation over all elements of $\mathcal{X}_{a,b}$, and derive strong stochastic transitivity for the distinguishabilities of players.

\begin{restatable}{lemma}{lemmaSSTforF}\label{lemma:SSTforF}
For a triplet of players $a\succ b \succ c$ it holds that  \[\E[X_{a,c}] \geq \max\{\E[X_{a,b}],\E[X_{b,c}]\}.\]
\end{restatable}
This also yields the following theorem, which paves the way for our reduction in what follows.

\begin{theorem}\label{cor:ZlargerThanDelta}
For any $a,b\in [n]$ such that $a\in A^*_k, b\notin A^*_k$ it holds that $\E[X_{a,b}]\geq \E[X_{k,k+1}]=\Delta$. Thus, if $\Delta>0$ and for a team $A$ it holds that $\E[X_{a,b}]\geq \Delta$ for every $a\in A, b\in[n]\setminus A$, then $A=A^*$.
\end{theorem}

\noindent\textbf{The reduction: }
We now outline the gap-dependent algorithm.
The results we have derived in Section \ref{sec:witnesses2} will allow us to deduce, with high probability, whether a distinguishability of a given pair of players is at least $\Delta$, and if so determine which is the better player.
Intuitively, this is done
by performing $\mathcal{O}(\frac{1}{\Delta^2})$ team duels.
We use $\E[X_{a,b}]$ as a proxy for the  distinguishability between two single players, $a,b$, taking advantage of the fact that if their relation is deducible, then $\E[X_{a,b}]\ne 0$ and in this case $\E[X_{a,b}]>0$ iff $a\succ b$. Similar the to dueling bandits setting, even though $|\E[X_{a,b}]|<\Delta$ for some pairs of players, identifying $A^*_k$ with high probability is possible.   

Since we cannot directly sample $X_{a,b}$, we will instead sample uniformly at random a triplet of sets, $(S,S',T)$ from $\mathcal{X}_{a,b}$.
Using $(S,S')(\in\mathcal{S}_{a,b})$ and $(S,T)(\in\mathcal{T}_{a,b})$, we can 
then perform all the duels required for an unbiased sample of $X_{a,b}(S,S',T)$, which is by itself a sampling of $X_{a,b}$. 
Given any dueling teams instance, we define a dueling bandits instance as follows: for every two players $a,b\in [n]$, we define the probability that $a$ wins in a (singles) duel against $b$ as
\begin{equation}\label{eq:reuduction}
P_{a,b}=1/2 + \E[X_{a,b}].
\end{equation}
Clearly, $1-P_{a,b}=P_{b,a}$, 
$P_{a,b}\in[0,1]$ and $P_{a,b}>1/2$ implies $a\succ b$. In addition, Theorem \ref{thm:theOneTrueLemmaToRuleAllWitnesses} implies that $a$ is better than $b$ in this dueling bandits instance iff $a\succ^* b$.
So whenever a dueling bandits algorithm is asking for a duel query, $(a,b)$, we can make an independent sample of $X_{a,b}$ by randomly drawing 
a triplet $(S,S',T)\in\mathcal{X}_{a,b}$ and 
returning a random sampling of $X_{a,b}(S,S',T)+1/2$. 
In cases where the realization of $X_{a,b}(S,S',T)+1/2$ 
is in $\{1/4,1/2,3/4\}$, we assign $a$ as the duel winner if the result of flipping a coin with bias $X_{a,b}(S,S',T)+1/2$ is $1$. We formalize this idea in the sub-procedure \textit{singlesDuel} (in the appendix), that simulates a duel for classical dueling bandits settings using team duels. 
Notice that, by Lemma \ref{lemma:SSTforF}, the probabilities $P_{a,b}$ defined in (\ref{eq:reuduction}) satisfy SST with respect to the total order among the players induced by the ground truth order $\succ$. 
In addition, the feedback of each single player duel we perform is time-invariant, thus all the non-parametric assumptions for dueling bandits settings apply here.
The reduction allows us to identify the top $k$ players using any dueling bandit algorithm with the same goal 
that works for total order on arms that satisfy SST, and a gap between the top $k$ and $k+1$ arms as assumptions. We formalize this in the following theorem. 
\begin{theorem}
Given any dueling teams instance with $n$ and $k$ (namely, $P_{A,B}$ for every two teams that hold strict total order, SST, and consistency), we have that the dueling bandit instance defined by (\ref{eq:reuduction}) satisfies SST with respect to the ground truth order among players $\succ$ and for any two players $a \succ b$ it holds that $P_{a,b} \geq 1/2$. Moreover, $P_{k,k+1} = 1/2 + \Delta$.
\end{theorem} 
Using the above theorem we can use any dueling bandit algorithm for top $k$ identification to solve our problem.
Mohajer et al. \cite{mohajer17} provide an algorithm that
returns the top $k$ players with probability exceeding $1-(\log n)^{-c_0}$ with sample complexity at most $c_1(n+k\log k)\frac{\max{(\log\log n,\log k)}}{\Delta_{k,k+1}^2}$ in expectation, where $c_0$ and $c_1$ are universal positive constants and $\Delta_{k,k+1}$ is the distinguishability between the $k$ and the $k+1$ best players (see Algorithm $2$ and Theorem $1$ in \cite{mohajer17}).

Ren et al. \cite{ren2020} show an algorithm that returns the top $k$ players with probability at least $1-\delta$ with sample complexity $\mathcal{O}(\sum_{i\in[n]}(\Delta_{i}^{-2}(\log(n/\delta)+\log\log\Delta_{i}^{-1}))$, where 
$\Delta_i=\mathbbm{1}_{i\succ {k+1}}\cdot \Delta_{i,k+1}+\mathbbm{1}_{k\succ i}\cdot\Delta_{k,i}$ and $k, {k+1}$ are the top $k$ and the top $k+1$ players, respectively (see Algorithm $5$ and Theorem $8$ in \cite{ren2020})\footnote{
We remark that 
\cite{ren2020} also assume Stochastic triangle inequality which we do not, however it is only used to derive a lower bound.}.
These algorithms, together with Theorem \ref{thm:theOneTrueLemmaToRuleAllWitnesses} allow us to derive the following theorem.

\begin{theorem}\label{thm:stochasticUpper}
There exists an algorithm that returns $A^*_k$  with probability exceeding $1-(\log n)^{-c_0}$ with sample complexity at most $c_1(n+k\log k)\frac{\max{(\log\log n,\log k)}}{\Delta^2}$ in expectation, where $c_0$ and $c_1$ are universal positive constants.

In addition, there exists an algorithm that returns $A^*_k$ with probability at least $1-\delta$ with sample complexity $\mathcal{O}(\sum_{i\in[n]}(\Delta_{i}^{-2}(\log(n/\delta)+\log\log\Delta_{i}^{-1}))$, where 
$\Delta_i=\mathbbm{1}_{i\succ {k+1}}\cdot \E[X_{i,k+1}]+\mathbbm{1}_{k\succ i}\cdot\E[X_{k,i}]$ and $i$ denotes the top $i$ players, thus $\Delta_i\geq \Delta$ for every $i\in[n]$.
\end{theorem}

\section{Deterministic Setting} \label{sec:deterministic}

In the previous section we showed the existence of algorithms that identify the top $k$ team with a number of duels that depends on $\Delta$. But what if $\Delta$ is very small or even $0$? 
One reason for that can be that 
all relevant probabilities are close to $1/2$. More precisely, 
$P_{\{k\}\cup S,\{k+1\}\cup S'}$, $P_{\{k+1\}\cup S,\{k\}\cup S'}$, $P_{\{k+1\}\cup S,T}$, and $P_{\{k\}\cup S,T}$ are very close to $1/2$ for all $(S,S',T) \in \mathcal{X}^*_{k,k+1}$. This 
might also occur 
in classic dueling bandits settings, when the target is to separate the top $k$ players from the rest (e.g., \cite{mohajer17,ren2020}). As a result, a gap between the top $k$ and $k+1$ players is often a parameter of the sample complexity in such settings. For these cases, our approach presented in the stochastic section very much resembles the current literature.

The other, more interesting reason for $\Delta$ to be small is when there exist only a small number of witnesses. This is in particular the case when the probability matrix contains only few distinct values, 
as for example when feedback is deterministic, i.e., $P_{A,B} \in \{0,1\}$.
Note that in this setting, $(S,T) \in \mathcal{T}_{a,b}$ is a witness if and only if $S \cup \{a\} \succ_{obs} T \succ_{obs} S \cup \{b\}$, and $(S,S') \in \mathcal{S}_{a,b}$ is a witness if and only if $S \cup \{a\} \succ_{obs} S'\cup \{b\}$ and $S'\cup \{a\} \succ_{obs} S \cup \{b\}$.
This follows as for
any tuple $(S,S',T)\in \mathcal{X}_{a,b}$ which is not a witness it holds that $\E[X_{a,b}(S,S',T)]=0$.
Moreover, it is 
possible to come up with deterministic instances where up to $(2k-1)^2$ pairs do not have any 
witness to distinguish them. 
To overcome this issue, in this section we design algorithms for the deterministic case that are independent of $\Delta$. 
In the appendix we show that these results can be extended to a slightly stochastic environment.

The limitation of the set of witnesses makes the problem of identifying a Condorcet winning team in the deterministic setting surprisingly nontrivial. 
For general total orders, a crucial difficulty lies in efficiently proving that a given team is indeed Condorcet winning. However, we are still able to 
get the following result:
\begin{theorem}\label{thm:gap-independent1}
For deterministic feedback, there exists an algorithm that performs $\mathcal{O}(kn\log(k) + k^2\log(k)2^{5k})$ duels and outputs a Condorcet winning team. 
\end{theorem}
For the natural special case of \emph{additive total orders} we obtain a significantly better upper bound.
A total order $\succ$ is \emph{additive total}, if there exist values for the players denoted by $v(x),x \in [n]$ such that $A \succ B$ iff $\sum_{a \in A}v(a)> \sum_{b \in B}v(b)$. In Section \ref{sec:additiveLinear} of the appendix we give a sufficient and necessary condition for a linear order to be additive. For additive linear orders 
we present an algorithm that identifies a Condorcet winning team after polynomial many duels and also outputs a proof. 
\begin{theorem}\label{thm:gap-independent}
For deterministic feedback and additive total orders, there exists an algorithm that finds a Condorcet winning team within $\mathcal{O}(kn\log(k) + k^5)$ duels.
\end{theorem}

Both algorithms rely on the same preprocessing procedure called \emph{ReducePlayers} which reduces the number of players from $n$ to $\mathcal{O}(k)$. At the heart of this procedure is a subroutine called \emph{Uncover}. 
After describing \emph{Uncover} and \emph{ReducePlayers}, we prove Theorem \ref{thm:gap-independent1}. Towards proving Theorem \ref{thm:gap-independent}, we 
introduce two more subroutines, namely \emph{NewCut} and \emph{Compare}, which are crucial for identifying and proving a Condorcet winning team within the smaller instance.
Finally, Algorithm \emph{CondorcetWinning} combines all components and proves Theorem \ref{thm:gap-independent}.

\medskip

\paragraph{The Uncover Subroutine}\label{subsec:uncover}

Given two disjoint teams $A \succ B$, the \emph{Uncover} subroutine finds a pair of players $a \in A$ and $b \in B$ and a subsets witness for their relation, i.e., an element from $\mathcal{S}^*_{a,b}$. To understand the idea of the subroutine, consider some arbitrary ordering of the elements in $A$ and $B$, respectively, i.e., $A = \{a_1,\dots, a_k\}$ and $B=\{b_1, \dots, b_k\}$. Then, iteratively exchange the elements $a_1$ and $b_1$, $a_2$ and $b_2$, resulting in sets $A_0 = A, B_0= B, A_1 = \{b_1,a_2,\dots,a_k\},  B_1 = \{a_1,b_2,\dots,b_k\}$, $A_2 = \{b_1,b_2,a_3,\dots,a_k\}$, and so on. Since $A_0 \succ B_0$ but $A_0=B_k \succ A_k=B_0$ holds, there needs to be some earliest point in time $i \leq k$ for which $B_i \succ A_i$ is true. This implies $a_i \succ b_i$ as 
 $(\{a_1, \dots a_{i-1}, b_{i+1}, \dots b_k\},\{ b_1, \dots , b_{i-1} , a_{i+1}, \dots ,a_k\})$ is a witness for this relation. 

While the above sketched subroutine is simple, it performs $k$ duels in the worst case. We refine this idea by a binary search approach, decreasing the number of duels to $\log(k)$.  

\begin{restatable}{lemma}{uncoverlemma}\label{lem:uncover}
Let $A$ and $B$ be two disjoint teams with $A \succ B$. After performing $\mathcal{O}(\log(k))$ duels, \emph{Uncover} returns $(a,b)$ with $a \in A$, $b \in B$ and $(S,S') \in \mathcal{S}^{*}_{a,b}$, and thus  $a\succ b$. 
\end{restatable}

We remark that Lemma 
\ref{lem:uncoverStrong} in the appendix  is a slightly stronger version of the above lemma, 
which allows us to partition $A$ and $B$ into two subsets each, $A=A^{(1)}\cup A^{(2)}$ and $B=B^{(1)}\cup B^{(2)}$. Under some circumstances, we can then guarantee that \emph{Uncover} reveals the pairwise comparison between two players $a \succ b$, where $a$ is from $A^{(1)}$ and $b$ is from $B^{(1)}$.

\medskip
\paragraph{Reducing the Number of Players to $\bm{\mathcal{O}(k)}$}\label{subsec:graph-algo}

The fact that we can eliminate some players from $[n]$ and still find (and prove) a Condorcet winning team is due to the following observation.

\begin{restatable}{observation}{lemobs}\label{obs:2k}
Let $\Sub \subseteq [n]$ such that $A_{2k}^* \subseteq \Sub$. Let $\W \subseteq \Sub$ be a team such that $\W \succ A$ for all teams $A \subseteq \Sub \setminus \W$. Then, $\W$ is a Condorcet winning team. 
\end{restatable}

The procedure \textit{ReducePlayers} reduces the set of players $[n]$ to some subset $\Sub \subseteq [n]$ guaranteeing that $A^*_{2k} \subseteq \Sub$ and $|\Sub| < 6k$. 
The algorithm maintains a dominance graph $D = (V,E)$ on the set of players. More precisely, the nodes of $D$ are the players, i.e., $V=[n]$, and there exists an arc from node $a$ to node $b$ if the algorithm has proven that $a \succ b$. The set $V_{<2k}$ is the subset of the players having an indegree smaller than $2k$ in $D$.
The high level idea of the algorithm is the following: It starts with the empty dominance graph  $D=([n],\emptyset)$. The algorithm then iteratively identifies pairwise relations of the players with help of \emph{Uncover} and adds the corresponding arcs to the graph. By adding more and more arcs to $D$, the set of nodes $V_{<\m}$ shrinks more and more while $A^*_{<\m} \subseteq V_{< \m}$ is always guaranteed. At some point, the algorithm cannot identify any more pairwise relations and returns $V_{<\m}$. 
How does the algorithm identify pairwise relations? At any point it tries to find a matching between $2k$ players, say $\{(a_1,b_1), \dots, (a_k,b_k)\}$ with the constraint that, for all $i \in [k]$, none of the arcs $(a_i,b_i)$ or $(b_i,a_i)$ is present within the graph $D$ yet. The algorithm ends when it cannot find such a matching anymore. We show that this only happens after $|V_{<\m}| < 6k$.

\begin{restatable}{lemma}{lemgraphalgocorrectness}\label{alg:graph_algo_correctness}
Given the set of players $[n]$, \emph{ReducePlayers} returns $\Sub \subseteq [n]$ with $|\Sub|\leq 6k-2$ and $A_{2k}^* \subseteq \Sub$. \emph{ReducePlayers} performs $\mathcal{O}(nk\log(k))$ duels and runs in time $\mathcal{O}(n^2k^2)$.
\end{restatable}
We can now prove Theorem \ref{thm:gap-independent1}.

\emph{Proof Sketch (of Theorem \ref{thm:gap-independent1})}. Let $D$ be the dominance graph at the end of \emph{ReducePlayers}. Then, the learner selects a $k$-sized subset of $V_{<2k}$, call it $\W$, with the property that there is no arc from any node in $V_{<2k} \setminus \W$ towards some node in $\W$. Then, the learner tests $\W$ against all possible teams containing players from $V_{<2k} \setminus \W$, which are $\mathcal{O}(2^{5k})$ many. If $\W$ wins all of these duels, then $\W$ is a Condorcet winning team by Observation \ref{obs:2k}. However, if there exists $A \succ \W$, then, by the choice of $\W$, there does not exist any arc from $A$ towards $\W$. Hence, by calling the subroutine \textit{Uncover} for two arbitrary orderings of $A=\{a_1, \dots, a_k\}$ and $\W=\{\hat{a}_1, \dots, \hat{a}_k\}$, the learner will identify one additional arc. This procedure can be repeated $\mathcal{O}(k^2)$ times and thus shows Theorem \ref{thm:gap-independent1}. \hfill \qedsymbol

\medskip

\paragraph{Subroutines NewCut and Compare} \label{subsec:NewCut}

The \emph{NewCut} subroutine takes as input a subset of the players $\Sub \subseteq [n]$, a pair $a,b \in \Sub$, and a witness proving that $a \succ b$, i.e., $(S,T') \in \mathcal{S}_{a,b}^* \cup \mathcal{T}_{a,b}^*$. That means, $T'$ can be either of size $k-1$ or $k$, and $S$ and $T'$ are not required to be subsets of $\Sub$. The subroutine outputs a partition of $\Sub$ into two non-empty sets $U$ and $L$ with $U \triangleright L$, which is short-hand notation for $u \succ \ell$ for any $u \in U$ and $\ell \in L$.
The subroutine starts by initiating the set $U=\{a\}$ and redefines $R=R \setminus \{a,b\}$. At all times, $U$ contains only players $u$ for which the algorithm has found a witness for $u \succ b$. 
These witnesses are stored in a list $\mathcal{W}$, and it is checked whether they can be modified to become witnesses for 
$x \succ b$ for any other element in $x \in \Sub$. This modification is done by applying permutations on the set of subsets of the players, similarly as done within the proof of Theorem \ref{thm:provableSingleRelationIffWitnessExists} and Lemma \ref{lemma:SSTforF}.  
If the algorithm finds a witness for $x \succ b$, then $x$ is added to $U$ and removed from $R$. Additionally, the new witness is stored in $\mathcal{W}$. This process ends when either $\Sub$ is empty or all witnesses in $\mathcal{W}$ have been checked. At this point it holds that $U \triangleright R \cup \{b\}$, and the algorithm returns $(U,L:= R \cup \{b\})$. 
 \begin{restatable}{lemma}{newcutcorrect}\label{prop:newcutcorrect}
Let $\Sub \subseteq [n]$, $a,b \in \Sub$ and $(S,T') \in \mathcal{S}_{a,b}^* \cup \mathcal{T}_{a,b}^*$. Then, $\mathrm{NewCut}(\Sub,(a,b),(S,T'))$ returns a partition of $\Sub$ into $U$ and $L$ such that $U \triangleright L$, $a \in U$ and $b \in L$. The number of duels performed by $\mathrm{NewCut}$ and its running time can be bounded by $\mathcal{O}(|\Sub|^2)$.  
\end{restatable}

\medskip
\subsection*{Additive linear orders}

From now on we assume additive linear orders. The \emph{compare} subroutine is crucial for obtaining upper bounds for differences between
values of players' subsets. It is used in the following situation. Let $(a,b)$ be a pair of players and $(S,S') \in \mathcal{S}^*_{a,b}$ be a witness for $a \succ b$. Then, it can be easily shown that $v(a) - v(b) > |v(S) - v(S')|$. We will be interested in the question whether a similar relation holds for two subsets of $S$ and $S'$, namely, $C \subseteq S$ and $D \subseteq S'$ of equal size. The \emph{compare} subroutine checks whether such a relation holds by performing two additional duels. 
If it returns \emph{True}, then $v(a)-v(b) > |v(C) - v(D)|$. Otherwise, there can be found a pair $c \in C$ and $d \in D$ and a witness for their relation by one call to the \emph{Uncover} subroutine. This observation is formalized below.

\begin{restatable}{lemma}{compare}\label{lem:compare}
Let $a \succ b$ be two players,  $(S,S') \in \mathcal{S}_{a,b}^*$ and $C \subseteq S, D \subseteq S'$ with $|C|=|D|$. If $\mathrm{Compare}((a,b),(S,S'),(C,D))$ returns \emph{True}, then $v(a)-v(b) > |v(C)-v(D)|$. Otherwise, one call to \emph{Uncover} returns $c \in C$ and $d \in D$ together with a witness for their relation. 
\end{restatable}

\definecolor{color1}{RGB}{0,101,189}
\definecolor{color2}{RGB}{162,173,0}
\definecolor{color3}{RGB}{227,114,34}
\definecolor{color4}{RGB}{202,033,063}
\definecolor{graphcolor}{RGB}{88,88,88}

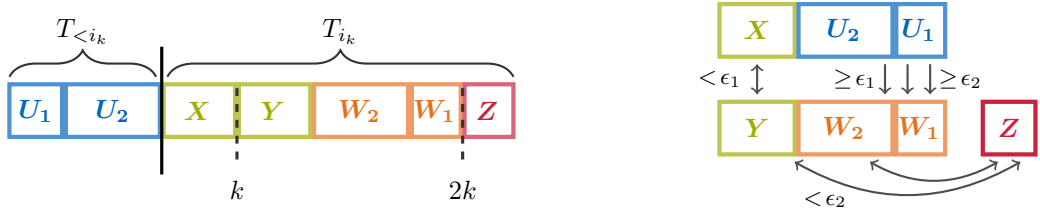
\begin{figure}
\begin{minipage}{.6\textwidth}
    \scalebox{1}{
    \begin{tikzpicture}
    \draw [draw=black] (0,0) rectangle (6.7,.7);
    \draw [draw=color1!70,ultra thick,] (0,0) rectangle (.7,.7);
    \node at (.35,.35) {\textcolor{color1}{$\bm{U_1}$}};
    \draw [draw=color1!70,ultra thick,] (.75,0) rectangle (2,.7);
    \node at (1.375,.35) {\textcolor{color1}{$\bm{U_2}$}};
    \draw [draw=color2!70,ultra thick,] (2.05,0) rectangle (3,.7);
        \node at (2.5,.35) {\textcolor{color2}{$\bm{X}$}};
    \draw [draw=color2!70,ultra thick,] (3.05,0) rectangle (4,.7);
    \node at (3.5,.35) {\textcolor{color2}{$\bm{Y}$}};
    \draw [draw=color3!70,ultra thick,] (4.05,0) rectangle (5.3,.7);
    \node at (4.65,.35) {\textcolor{color3}{$\bm{W_2}$}};
    \draw [draw=color3!70,ultra thick,] (5.35,0) rectangle (6,.7);
    \node at (5.65,.35) {\textcolor{color3}{$\bm{W_1}$}};
    \draw [draw=color4!70,ultra thick,] (6.05,0) rectangle (6.7,.7);
    \node at (6.35,.35) {\textcolor{color4}{$\bm{Z}$}};
    \draw[very thick] (2.025,-.5) -- (2.025,1.2);
    \draw[very thick,black!80,dashed] (3.025,-.3) -- (3.025,.7);
    \draw[very thick,black!80,dashed] (6.025,-.3) -- (6.025,.7);
    \node at (3.025,-.7) {$k$};
    \node at (6.025,-.7) {$2k$};
    \draw [decorate,thick,black!80, decoration={brace,amplitude=10pt},yshift=-.2cm]
(0,1) -- (1.95,1) node [black,midway,yshift=0.6cm] 
{ $T_{<i_k}$};
\draw [decorate,thick, black!80,decoration={brace,amplitude=10pt},yshift=-.2cm]
(2.1,1) -- (6.7,1) node [black,midway,yshift=0.6cm] 
{$T_{i_k}$};
    \end{tikzpicture}}
    \end{minipage}
    \begin{minipage}{0.35 \textwidth}
    \begin{tikzpicture}
    \draw [draw=black] (0,0) rectangle (3,.7);
    \draw [draw=black] (0,1.3) rectangle (3,2);
    \draw [draw=color4,ultra thick] (3.5,0) rectangle (4.2,.7);
    \node at (3.85,.35) {\textcolor{color4}{$\bm{Z}$}};
    \draw [draw=color2!70,ultra thick] (0,0) rectangle (1,.7);
    \node at (.5,.35) {\textcolor{color2}{$\bm{Y}$}};
    \draw [draw=color2!70,ultra thick] (0,1.3) rectangle (1,2);
    \node at (.5,1.65){\textcolor{color2}{$\bm{X}$}};
    \draw [draw=color1!70,ultra thick] (1.05,1.3) rectangle (2.3,2);
    \draw [draw=color3!70,ultra thick] (1.05,0) rectangle (2.3,.7);
    \node at (1.65,.35) {\textcolor{color3}{$\bm{W_2}$}};
    \node at (1.65,1.65) {\textcolor{color1}{$\bm{U_2}$}};
    \draw [draw=color1!70,ultra thick] (2.35,1.3) rectangle (3,2);
    \draw [draw=color3!70,ultra thick] (2.35,0) rectangle (3,.7);
    \node at (2.65,.35) {\textcolor{color3}{$\bm{W_1}$}};
    \node at (2.65,1.65) {\textcolor{color1}{$\bm{U_1}$}};
    \draw[<->,thick, black!70] (0.5,.8) -- (0.5,1.2);
    \node at (0,1){\small$< \!\epsilon_1$}; 
    \draw[<-,thick, black!70] (2.5,.8) -- (2.5,1.2);
    \draw[<-,thick, black!70] (2.8,.8) -- (2.8,1.2);
    \draw[<-,thick, black!70] (2.2,.8) -- (2.2,1.2);
    \node at (1.8,1){\small$\geq \!\epsilon_1$}; 
    \node at (3.2,1){\small$\geq \!\epsilon_2$}; 
    \draw[<->,thick, black!70] (3.7,-.1) to [bend left] (2,-.1);
    \draw[<->,thick, black!70] (4,-.1) to [bend left] (1,-.1);
     \node at (1.4,-0.6){\small$< \!\epsilon_2$}; 
    \end{tikzpicture}
    \end{minipage}\caption{Illustration of the proof technique of algorithm \emph{CondorcetWinning1}. In the left illustration, the solid black line indicates that all players left to it were proven to be better than all players right to it. 
    The dashed line marked with ``$k$'' indicates that the sets to its left contain $k$ players in total. However, this line does not indicate proven relations, i.e., players from $X$ are not necessarily 
     better than players from $Y$. 
    The right figure illustrates the proof for $X \cup U_1 \cup U_2$ being Condorcet winning.} \label{fig:1} 
\end{figure}

\medskip

\paragraph{Algorithm CondorcetWinning} \label{subsec:Condorcet}

The algorithm maintains a partition of the players into a weak ordering, i.e., $\mathscr{T} = \{T_1, \dots, T_{\ell}\}$ with $T_1 \triangleright T_2 \triangleright \dots \triangleright T_{\ell}$. We introduce the short-hand notation $T_{\leq j} = \bigcup_{m \in [j]} T_m$ and $T_{< j} = \bigcup_{m\in[j-1]} T_m$.
After the application of the preprocessing procedure \emph{ReducePlayers}, this partition consists of one set, namely $\mathscr{T} = \{T_1\}$, where $|T_1| \in \mathcal{O}(k)$ and $A^*_{2k} \subseteq T_1$.
At any point in the execution of the algorithm, we are especially interested in two indices, namely $i_k \in [\ell]$ such that $|T_{<i_k}|<k<|T_{\leq i_k}|$ and similarly  $i_{2k} \in [\ell]$ such that $|T_{<i_{2k}}|<2k<|T_{\leq i_{2k}}|.$ \footnote{In case one of these indices does not exist,
it implies that we have either identified the set $A^*_k$ or $A^*_{2k}$. In the first case we have found a Condorcet winning team and in the second case Observation \ref{obs:2k} implies that we can find one by performing one additional duel. For the sake of brevity we disregard this case from now on.} Observe that all players from $T_{<i_k}$ are guaranteed to be among the top-k players. On the other hand, among the players from $T_{i_k}$ some belong to $A^*_k$ and others do not. 
The main idea of the algorithm is then the following: Take a prefix of $\mathscr{T}$ of size $k$, i.e., this team contains the set of players $T_{<i_k}$ and is a subset of the players in $T_{\leq i_{k}}$, and either prove that this prefix is a Condorcet Winning team, or refine the partition $\mathscr{T}$ 
and repeat the process. The refinement is done by splitting one element of $\mathscr{T}$, say $T_i$, into two non-empty sets, $T_i^1 \triangleright T_i^2$, and re-indexing the sets within $\mathscr{T}$. Clearly, this increases the number of sets within the partition $\mathscr{T}$ by one.

We provide two different algorithms, namely \emph{CondorcetWinning1} for the case $i_k = i_{2k}$ and \emph{CondorcetWinning2} when $i_k \neq i_{2k}$. Unsurprisingly, the latter case requires a strictly less sophisticated approach, which is why we focus on \emph{CondorcetWinning1} in the following.

The algorithm starts by partitioning the set $T_{< i_{k}}$ into two sets $U_1$ and $U_2$, where $U_1$ is a prefix of $T_{<i_k}$ of size $|T_{\leq i_k}|-2k$. It partitions the set $T_{i_k}$ into five sets $X,Y,W_1,W_2,$ and $Z$. In particular it is known that $(U_1 \cup U_2) \triangleright (X \cup Y \cup W_1 \cup W_2 \cup Z)$ but no relation among any pair in $T_{i_k}$ is known. Regarding the sizes of the sets it holds that $|U_i| = |W_i|$ for $i \in \{1,2\}$, $|X|=|Y| = k - |U_1|-|U_2|$ and $|U_1| = |Z|$. The main aim of the algorithm will be to define $0< \epsilon_1 < \epsilon_2$ and prove that the following statements are true:
\begin{enumerate}[label=(\roman*)]
    \item $|v(X) - v(Y)| < \epsilon_1$
    \item $|v(a) - v(b)| < \epsilon_2$ for all $a \in Y \cup W_1 \cup W_2$ and $b \in Z$, and 
    \item there exist $u_1, \dots, u_{|Z|+1} \in U_1 \cup U_2$ as well as $w_1, \dots, w_{|Z|+1} \in W_1 \cup W_2$ such that 
    \begin{enumerate}[label=(\alph*)]
        \item $v(u_1) - v(w_1) \geq \epsilon_1$ and
        \item $v(u_i) - v(w_i) \geq \epsilon_2$ for all $i \in \{2, \dots, |Z|+1\}$.
    \end{enumerate}
\end{enumerate}
With these three statements we can show that $U_1 \cup U_2 \cup X$ is a Condorcet winning team. More precisely, one can show that $v(U_1 \cup U_2 \cup X) - v(W_1 \cup W_2 \cup Y) > |Z| \cdot \epsilon_2$ and $v(W_1 \cup W_2 \cup Y) - v(B^*) > - |Z| \cdot \epsilon_2$, where $B^*$ is the best response\footnote{We say that $B^*$ is a best response towards $U_1 \cup U_2 \cup X$, if $B^*$ contains the best $k$ players from $[n] \setminus (U_1 \cup U_2 \cup X)$.} towards $U_1 \cup U_2 \cup X$. See Figure \ref{fig:1} for an illustration of the argument. 

It remains to sketch how the algorithm defines $\epsilon_1,\epsilon_2$ and proves $(i)-(iii)$. %, 
For simplicity assume $U_1 \triangleright U_2$. The algorithm then attempts to do the following steps: (1) Find a witness for players $\Bar{u} \in U_2$ and $\Bar{w} \in W_2$, using \emph{Uncover}. (2) Use \emph{Compare}, to prove that $|v(X) - v(Y)| < v(\Bar{u}) - v(\Bar{w})$ and $|v(a) - v(b)| < v(\Bar{u}) - v(\Bar{w})$ holds for all players $a \in W_1 \cup W_2 \cup Y$ and $b \in Z$. (3) Repeat step (2) by replacing $\Bar{w}$ with any player of $W_1$. If one of the steps (1)-(3) fails, we show that the partition $\mathscr{T}$ can be refined. Otherwise, we show that $(i)-(iii)$ hold for $\epsilon_1 = v(\Bar{u}) - v(w^*_1)$ and $\epsilon_2 = v(\Bar{u}) - v(w_2^*)$, where $w_1^*$ and $w_2^*$ are the best and second best players from $W_1 \cup \{\Bar{w}\}$, respectively. The following Lemma concludes the proof sketch of Theorem \ref{thm:gap-independent}.

\begin{restatable}{lemma}{lemCondorcetWinning} \label{lemCondorcetWinning}
For every instance with $\mathcal{O}(k)$ players, after performing $\mathcal{O}(k^5)$ many duels, \emph{CondorcetWinning1} has identified a Condorcet winning team. \emph{CondorcetWinning2} identifies a Condorcet winning team after $\mathcal{O}(k^2\log(k))$ duels. 
\end{restatable}

%\vspace{.5cm}
\section{Extensions and Discussion} \label{sec:extensions}

In the following we discuss several implications of our results as well as directions for future work. 

\textbf{Checking Condorcet winners beyond additive linear orders} As we have briefly discussed within Section \ref{sec:deterministic},
the question how many duels are necessary to prove (or disprove) that a given team is a Condorcet winning team (even in an instance with $3k$ players) remains open for total orders that are not additive linear. A polynomial upper bound for this number would, together with our algorithm of Theorem \ref{thm:gap-independent1}, yield an algorithm with a polynomial number of duels. 
We formalize this observation within the following Corollary. 

\begin{corollary}
Let $q$ be the number of duels required to check whether a given team is a Condorcet winning team within an instance with $\mathcal{O}(k)$ players. Then, there exists an algorithm that identifies a Condorcet winning team within $\mathcal{O}(kn \log(k) + k^2 log(k) q )$ duels. 
\end{corollary}
\textbf{Lower Bounds}
For the stochastic and the deterministic setting, there exists a lower bound of $n-2k$ duels in order to identify a Condorcet winning team: Consider an adversary that fixes, over time, a reverse lexicographical order, i.e., a duel is decided against the worst player participating. When the algorithm performs its first duel, the adversary picks an arbitrary player from the duel, makes him player $n$ and answer the query accordingly. Then, whenever the algorithm performs a duel containing a player which has already been fixed, the adversary decides the duel against the worst fixed player participating. Otherwise, he picks an arbitrary player from the duel and fixes him to become player $n-t$, where $t$ is the number of so far fixed players. As long as $t<n-2k$, the algorithm cannot not identify a Condorcet winning team. 
\begin{theorem}
Any algorithm that identifies a Condorcet winning team performs at least  $n-2k$ duels. 
\end{theorem}

Note that the above theorem is tight in the dependency on $n$, for small team size $k=o(n)$.
Deriving tighter lower bounds for our team setting, especially the dependency on the team size, is an interesting question for future work. 

\textbf{Regret Bound}  
In this paper we provided algorithms to identify, with high probability, a Condorcet winning team. However, there exist other  performance metrics for online learning theory, which apply in  particular in MAB and dueling bandits. 

As there exists more than a single Condorcet winning team, it is reasonable to define regret w.r.t. the best possible team, i.e., $A^*_k$  for our setting, i.e.,
\[
R_T=\sum_{t=1}^T \min{\{P_{A^*_k,A_t}-1/2,P_{A^*_k,B_t}-1/2\}},
\]
where $(A_t,B_t)$ is the selected duel at time $t$ 
and $T$ is the time horizon\footnote{ This definition is based on weak regret for dueling bandits, as defined in \citetAppendix{Yue2012}.}.\\
 Using the second part of Theorem \ref{thm:stochasticUpper}, one can choose $\delta=1/(Tn)$ and 
achieve a regret bound 
of 
\[
R_T=(1-(Tn)^{-1})\cdot n(\Delta^{-2}(\log(T)+\log\log \Delta^{-1})+(Tn)^{-1}=\mathcal{O}(n(\Delta^{-2}(\log(T)+\log\log \Delta^{-1})).
\]
This follows from the SST of the distinguibilities (Lemma \ref{lemma:SSTforF}) implies $\Delta_i\geq \Delta$ for all $i \in [n]$. 

\section{Acknowledgments}
This project has received funding from the European Research Council (ERC) under the European Union’s Horizon 2020 research and innovation program (grant agreement No. 882396), the Israel Science Foundation (grant number 993/17), the Yandex Initiative for Machine Learning at Tel Aviv University,  the Deutsche Forschungsgemeinschaft under grant BR 4744/2-1, and the Ariane de Rothschild Women Doctoral Program.

This paper is dedicated to Hunter, a dear friend who passed away May 31, 2021. 

\newpage

\bibliography{refs}
\bibliographystyle{apa-good}

\newpage
\appendix

\onecolumn

\section*{Appendix}

\section{Extended Version and Proofs of Section \ref{sec:witnesses2}}\label{sec:witnessesFull}

Within the main text, we covered two different types of witnesses for single players relations. In this section, we show that whenever a relation between single players can be proven from  observable duels in our setting, there exists at least one type of witness for it. 
For the convince of the reader, we recall the definitions mentioned in the main text in a comprehensive manner, provide more explanations and some examples.

\noindent\textbf{Possible Witnesses}
For two players $a$ and $b$ we define $\mathcal{S}_{a,b}$ as the set of pairs of disjoint $k-1$ sized subsets of players from $[n]\setminus \{a,b\}$, i.e., 
\[\mathcal{S}_{a,b} = \{(S,S') \mid S,S' \subseteq [n] \setminus \{a,b\}, S \cap S' = \emptyset, |S|=|S'|=k-1\},\] 

and $\mathcal{T}_{a,b}$ as the set of disjoint $k-1$ sized subset $S$ and a team $T$ pair from $[n]\setminus \{a,b\}$, i.e., 
\[\mathcal{T}_{a,b} = \{(S,T) \mid S,T \subseteq [n] \setminus \{a,b\}, S \cap T = \emptyset, |S|=k-1, |T|=k\}.\]

\begin{definition}[Witnesses and Witnesses sets]\label{def:witnessesTypes}
A \emph{witness for $a\succ b$} is one of the following types:
(i) \emph{Subsets}: A pair of disjoint subsets $(S,S')\in \mathcal{S}_{a,b}$ such that
    \[
P_{\{a\}\cup S, \{b\}\cup S'}> P_{\{b\}\cup S\succ \{a\}\cup S'}.
\]
We denote the set of all  \emph{subsets witnesses for $a\succ b$} by $\mathcal{S}_{a,b}^*$.\\
(ii) \emph{Subset-Team}: $(S,T)\in \mathcal{T}_{a,b}$, such that 
\[
P_{\{a\}\cup S, T}>P_{\{b\}\cup S,T}.
\]
We denote the set of all  \emph{subset-team witnesses for $a\succ b$} by $\mathcal{T}_{a,b}^*$.

\end{definition}
In case we find a witness, we can use it to compare players as follows.
\begin{restatable}{lemma}{obsWitness}\label{obs:witness}
If there exists a pair $(S,S')\in \mathcal{S}^*_{a,b}$, or a pair $(S,T)\in \mathcal{T}^*_{a,b}$, then  $a\succ b$.
\end{restatable} 
\begin{proof}
First, consider the existence of $(S,S')\in \mathcal{S}_{a,b}^*$.\\
Hence 
\[
(*)\;P_{S\cup\{a\}, S' \cup \{b\}}>P_{S \cup \{b\} , S' \cup \{a\}}
\]
Assume for contradiction that $b\succ a$. Consistency implies $S \cup \{b\}\succ S \cup \{a\}$ and $S' \cup \{b\}\succ S' \cup \{a\}$.

Adding up the two implications from the witness definition and SST, we have
\[
P_{S\cup\{a\}, S' \cup \{b\}}>_{(*)} P_{S \cup \{b\} , S' \cup \{a\}}>_{b\succ a} P_{S \cup \{a\} , S' \cup \{a\}}>_{b\succ a}P_{S \cup \{a\} , S' \cup \{b\}},
\]
Which is a contradiction.

Now, consider the existence of  $(S,T)\in \mathcal{T}^*_{a,b}$.
We have that 
\[
(**)\;P_{S\cup\{a\}, T}>P_{S \cup \{b\} , T}
\]
Assume for contradiction that $b\succ a$. Consistency implies $S \cup \{b\}\succ S \cup \{a\}$.
\[
P_{S\cup\{b\}, T}>_{b\succ a} P_{S \cup \{a\} , T}>_{(**)}P_{S \cup \{b\} ,T},
\]
Which is a contradiction.
\end{proof}

Note that while the above lemma implies a sufficient condition for $a\succ b$,  there is no guarantee that for every $a\succ b$ there exists a witness that proves it, as it requires disjoint subsets. For example, consider a lexicographical order among teams with $n=4,k=2$ 
with uniform noise, e.g. when $P_{A,B}=0.6$ for all teams $A\succ B$. It follows from consistency and $12\succ 23$ that $2\succ 3$, but there is no witness for that. Moreover, even if we execute each of the $3$ possible duels enough to estimate correctly that $P_{12,34}=P_{13,24}=P_{14,23}=0.6$ there is no way to distinguish between the second and third best players.
In what follows we formalize this intuition, showing that if single players relation is provable then one of the aforementioned witnesses types exists for it.

Next, we recall the Observable relation and the set $\mathcal{C}_{obs}$.

\noindent\textbf{Observable relation} 
Let $\obs$ denote the relation between every two disjoint teams, i.e.,
\[
A\succ_{obs}B \; \iff \; A\succ B,\; |A|=|B|=k,\; A\cap B=\emptyset,\; A, B\subseteq [n].
\]
Namely the relation $\obs$ is deducible from valid duels \footnote{Notice that technically, $\obs$ is not defined on pairs of different teams which are not disjoint, and therefore not even a partial order on teams (e.g., we have that $\{a,b\}\obs \{c,d\}\obs \{a,e\}$ but $\{a,b\}\nsucc_{obs}  \{a,e\}$ as they share a player and the duel $(\{a,b\},\{a,e\})$ is not observable.).}.

In what follows, we elaborate more on the definition of $\mathcal{C}_{obs}$ by defining first a set for Compatible winning probabilities.

\noindent\textbf{Compatible winning probabilities} 
Let $\mathbb{P}_{obs}$ be the set of all tuples $(P',\succ')$, where $P'$ are the winning probability matrices for teams, i.e., $P'=(P'_{A,B})_{A\ne B, |A|=|B|=k, A,B\in [n]}\in [0,1]^{n\choose k}\times [0,1]^{n\choose k}$, and $\succ'$ is a consistent total order on the teams such that:
\begin{enumerate}
    \item For every pair of disjoint teams $(A,B)$ the winning probability matrix $P'$ has the same winning probability as the ground truth $P$, i.e.,  $A\cap B=\emptyset$ implies $P'_{A,B}=P_{A,B}$.
    \item It holds that $P'_{A,B}=1/2$ iff $A=B$. 
    \item $P'_{A,B} > 1/2$ if and only if $A \succ' B$. 
    \item $P'$ satisfies SST w.r.t. $\succ'$.
\end{enumerate}
 Namely, $\mathbb{P}_{obs}$ contains all tuples $(P',\succ')$ that do not contradict the winning probabilities the learner can observe and our assumptions.

\noindent\textbf{Compatible relations} 
Let $\mathcal{C}_{obs}$ be the set of all total orders $\succ'$ for which there exists $(P',\succ') \in \mathbb{P}_{obs}$. 
Notice that by the definition of $\mathbb{P}_{obs}$, we know that $\succ'$ satisfy consistency and in particular it holds that $A\succ' B$ for every disjoint teams $(A,B)$  with $A\obs B$. Namely, $\mathcal{C}_{obs}$ is the sets of all possible total orders that could explain the results of the observable duels.

We remark that it follows directly from the definition of $\mathbb{P}_{obs}$ that $(P,\succ) \in \mathbb{P}_{obs}$, where $P$ is the ground truth winning probability matrix and $\succ$ the ground truth total order. 
Because of this, it also holds that $\succ$ is in $\mathcal{C}_{obs}$.
To illustrate that $\succ$ is typically not the only total order in $\mathcal{C}_{obs}$, we provide the following example.
\begin{example}
For $n=5, k=2$, consider 
the lexicographic order 
, i.e., $\{1,2\} \succ \{1,3\} \succ \{1,4\}  \succ \{1,5\} \succ \{2,3\} \succ \{2,4\} \succ \{2,5\} \succ \{3,4\} \succ \{3,5\}  \succ \{4,5\} $ and assume $P_{A,B}=0.6$ iff $A\succ B$ (equivalently $P_{A,B}=0.4$ iff $B\succ A$).
Then, we have that 
\[
 A\obs B \iff
\begin{cases}
1\in A \text{, or}\\
1\notin A\cup B, 2\in A.
\end{cases}
\]

While $\succ \in \mathcal{C}_{obs}$, there are  other consistent total orders in $\mathcal{C}_{obs}$, such as $\{1,2\} \succ' \{1,5\} \succ' \{1,4\}  \succ' \{1,3\} \succ' \{2,5\} \succ' \{2,4\} \succ' \{2,3\} \succ' \{5,4\} \succ' \{5,3\}  \succ' \{4,3\} $ (the order $\succ'$ is obtained by 
 swapping players $3$ and $5$ in $\succ$). Similarly, the probability matrices $P_{A,B}=0.6$ for all $A\succ B$,
 (the ground truth), but $P^1_{A,B}=0.7\; \forall\; A\succ B$ and $P^2_{A,B}=0.6 \;\forall\; A\succ' B$ are also in $\mathbb{P}$.
\end{example}

 We now recall the definition of the deducible relation, $\succ^*$ for both teams and single players, where the latter definition is a combination of the former and single players consistency.
 
 The intuition behind these definitions is that a relation can be deducible (proven) by team duels if any ``reasonable'' total order that could possibly be the ground order agree on this relation. We stress that both $\mathbb{P}_{obs}$ and $\mathcal{C}_{obs}$ are strictly for analysis, as we do not need to explicitly calculate them.
 
 \begin{definition}
 Team $A$ is deducibly better than  a different team $B$, denoted by $A\succ^* B$ (using team duels), if $A\succ' B$ for all $\succ'\in \mathcal{C}_{obs}$. 
 \end{definition}
 
  \begin{definition}
 Player $a$ is deducibly better than player $b$, denoted by $a\succ^* b$, if $\{a\}\cup S\succ' \{b\}\cup S$ for all $\succ'\in \mathcal{C}_{obs}$. 
 \end{definition}

 We continue with an example for relations that $\succ^*$ must satisfy.
Suppose the learner has observed that $\{a,c\}\obs \{b,d\}\obs \{a,e\}\obs \{c,d\}$. Since all the relations $\succ \in \mathcal{C}_{obs}$ satisfy transitivity, it follows that $\{b,d\}\succ^* \{c,d\}$, $\{a,c\}\succ^* \{a,e\}$, and $\{a,c\}\succ^* \{c,d\}$. As each $\succ \mathcal{C}_{obs}$ also satisfies single players consistency, we deduce $b\succ^* c$, $c\succ^* e$ and $a\succ^* d$, respectively. Applying single players consistency again, we can get, for example, $\{a,b\}\succ^* \{a,c\}\succ^* \{a,e\}\succ^* \{d,e\}$ (using $b\succ^* c$, $c\succ^* e$ and $a\succ^* d$, respectively).

Intuitively, what we will show in Theorem \ref{thm:provableSingleRelationIffWitnessExists} is that for every pair of players that one is provably better than the another there exists a witness for it, thus there is a short proof with which the learner can verify their relation with $O(1)$ queries in the deterministic case. 
Before we start proving the Theorem \ref{thm:provableSingleRelationIffWitnessExists} we prove the following helpful lemma. 

\begin{lemma}\label{lem:helper1}
Let $\succ \in C_{obs}$ and $P$ be a corresponding probability matrix satisfying SST.

Let $a,b \in [n]$ with $a \succ b$. Then, the following holds true: 
\begin{enumerate}
    \item Let $(S,S') \in \mathcal{S}_{a,b}$, then $P_{\{a\} \cup S, \{b\} \cup S'} \geq P_{\{b\} \cup S, \{a\} \cup S'}$.
    \item Let $(S,T) \in \mathcal{T}_{a,b}$, then
    $P_{\{a\} \cup S,T} \geq P_{\{b\} \cup S, T}$.
\end{enumerate}
 
\end{lemma}

\begin{proof}
1. We start by proving that for every $(S,S')\in\mathcal{S}_{a,b}$ it holds that $$P_{\{a\}\cup S,\{b\}\cup S'}\geq P_{\{b\}\cup S,\{a\}\cup S'}$$  by exhaustion.
\begin{enumerate}[label=(\alph*)]
\item If $S \cup \{a\} \succ S' \cup \{b\}$ and $S' \cup \{a\} \succ S \cup \{b\}$ then it follows that $1/2<P_{\{a\}\cup S,\{b\}\cup S'},P_{\{a\}\cup S',\{b\}\cup S}$ and therefore 
\[
P_{\{a\}\cup S,\{b\}\cup S'}>1/2>1-P_{\{a\}\cup S',\{b\}\cup S}=P_{\{b\}\cup S,\{a\}\cup S'}.
\]
\item If (a) does not hold, then it follows that either of the following holds true:
\begin{enumerate}[label=(\roman*)]
\item  $\{b\}\cup S\succ \{a\}\cup S'$ (and $\{a\}\cup S\succ \{b\}\cup S'$ as $b\nsucc a$).\\
From single players consistency of $\succ$ we have that 
\[
\{a\}\cup S\succ\{b\}\cup S\succ \{a\}\cup S'\succ\{b\}\cup S'
\]
Applying SST, we have that 
\[
P_{\{a\}\cup S, \{b\}\cup S'}\geq P_{\{a\}\cup S, \{a\}\cup S'}\geq P_{\{b\}\cup S,\{a\}\cup S'}.
\]
\item $\{b\}\cup S'\succ \{a\}\cup S$ (and $\{a\}\cup S'\succ\{b\}\cup S$ as $b\nsucc a$).\\
From consistency, we have that 
\[
\{a\}\cup S'\succ\{b\}\cup S'\succ \{a\}\cup S\succ\{b\}\cup S
\]
Applying SST, we have that 
\[
 P_{\{a\}\cup S', \{b\}\cup S}\geq P_{\{a\}\cup S', \{a\}\cup S}\geq P_{\{b\}\cup S',\{a\}\cup S}.
\]
Therefore
\[
1-P_{\{b\}\cup S',\{a\}\cup S}\geq 1-P_{\{a\}\cup S', \{b\}\cup S}.
\]
Applying $P_{A,B}=1-P_{B,A}$ for every $A,B\in [n]$,
\begin{equation*}\label{eq:sst1}
 P_{\{a\}\cup S,\{b\}\cup S'}\geq P_{\{b\}\cup S,\{a\}\cup S'}.
\end{equation*}
\item The case that $\{b\}\cup S'\succ \{a\}\cup S$ and $\{b\}\cup S'\succ\{a\}\cup S$ cannot hold as it would imply $b\succ a$ which is a contradiction to $a\succ b$, as $\succ$ being a consistent total order yields a total order on players.
\end{enumerate}
\end{enumerate}

2. Strict total order on teams together with consistency implies  that  either of the following holds: (a) $\{a\}\cup S\succ \{b\}\cup S \succ T$,  (b)$\{a\}\cup S \succ T\succ \{b\}\cup S$, or (c) $T\succ \{a\}\cup S \succ \{b\}\cup S$. Applying SST on (a) and (c) proves the claim, and if (b) holds we have
\begin{equation*}\label{eq:sst2}
    P_{\{a\}\cup S,T}>1/2> P_{\{b\}\cup S,T}.
\end{equation*}
\end{proof}
We note that the left to right direction in the following sentence is very similar to Lemma \ref{obs:witness} and their proofs are equivalent, however for completeness we provide a full proof here as well.
\thmcharacterization*
%\lccom{should we spilit this thm's proof to couple of lemmas?}
%\uskcom{I think I would prefer to keep it this way, because then the reader who reads the main text (and only shortly looks into the appendix) directly sees where the proof of our main theorem is. Generally, I prefer to only have lemmas when they are used more than once. }\lccom{OK}

%
\begin{proof}
We start with the direction from right to left, i.e.,  $\mathcal{S}^*_{a,b} \cup \mathcal{T}^*_{a,b} \neq \emptyset$ implies $a \succ^* b$. %This can be shown with the help of SST, the fact that $\succ$ is a consistent strict total order, and an exhaustive case distinction:

First, consider $(S,S') \in \mathcal{S}^*_{a,b}$ and assume for contradiction that $a \succ^* b$ does not hold. That is, there exists $\succ' \in \mathcal{C}_{obs}$ and $P' \in \mathbb{P}_{obs}$ such that $b \succ' a$, %let 
and $P'$ is a %denote the 
corresponding winning probability matrix. 

By Lemma \ref{lem:helper1} and the definition of $\mathbb{P}_{obs}$ it follows that 
\[
P_{S \cup \{b\},S' \cup \{a\}} = P'_{S \cup \{b\},S' \cup \{a\}} \geq P'_{S \cup \{a\},S' \cup \{b\}} = P_{S \cup \{a\},S' \cup \{b\}}
\]
holds, as the teams are disjoint. This is a contradiction to $(S,S') \in \mathcal{S}_{a,b}^*$.

Similarly, let $(S,T) \in \mathcal{T}^*_{a,b}$ and assume for contradiction that $a \succ^* b$ does not hold. That is, there exists $\succ' \in \mathcal{C}_{obs}$ and $P' \in \mathbb{P}_{obs}$ such that $b \succ' a$, %let 
and $P'$ is a %denote the 
corresponding winning probability matrix. %there exists $\succ' \in \mathcal{C}_{obs}$ such that $b \succ' a$, let $P'$ be the corresponding probability matrix.
By Lemma \ref{lem:helper1} and the definition of $\mathbb{P}_{obs}$ it follows that 
\[
P_{S \cup \{b\},T} = P'_{S \cup \{b\},T} \geq P'_{S \cup \{a\},T} = P_{S \cup \{a\},T}
\]
holds, as the teams are disjoint. This is a contradiction to $(S,T) \in \mathcal{T}_{a,b}^*$. 

We turn to the direction from left to right, i.e. that $a \succ^* b$ yields $S_{a,b}^* \cup \mathcal{T}_{a,b}^* \neq \emptyset$. We start by defining $\mathcal{D}_a$ as the set of observable duels $(A,B)$ such that $a \in A$. Moreover, we define a permutation $\pi$ on the set of players, which simply exchanges the players $a$ and $b$ when present. More precisely,
\[\pi(S) = \begin{cases} S \setminus \{a\} \cup \{b\} & \text{ if } a \in S, b \not\in S \\
S \setminus \{b\} \cup \{a\} & \text{ if } b \in S, a \not\in S \\
S & \text{ else.}
\end{cases}\] 
% \lccom{I changed the definition of $\pi$ to apply on a subset of players because later we apply it on sorbets which are not teams}

We claim that $a \succ^* b$ implies 
\begin{align}\label{eq:PIsGeq}
    P_{A,B} \geq P_{\pi(A),\pi(B)} \text{ for all } (A,B) \in \mathcal{D}_a
\end{align}
($P$ is the ground truth winning probability matrix). To see why, we first define 
\begin{align*}
    \mathcal{D}_a^1 &= \{(A,B) \in \mathcal{D}_a \mid b \in A\} \\
    \mathcal{D}_a^2 &= \{(A,B) \in \mathcal{D}_a \mid b \in B\} \\
    \mathcal{D}_a^3 &= \{(A,B) \in \mathcal{D}_a \mid b \notin A \cup B\}. \\
\end{align*} 
Notice that 
\begin{align}\label{eq:UnionD}
    \mathcal{D}_a=\mathcal{D}_a^1\cup \mathcal{D}_a^2\cup \mathcal{D}_a^3
\end{align}
When $(A,B) \in \mathcal{D}_a^1$, then  $(\pi(A),\pi(B))=(A,B)$ and $P_{A,B} = P_{\pi(A),\pi(B)}$.

When $(A,B) \in \mathcal{D}_a^2$, then $(A\setminus \{a\},B \setminus \{b\}) \in \mathcal{S}_{a,b}$,  and $P_{A,B} \geq P_{A\setminus \{a\}\cup \{b\},B \setminus \{b\}\cup \{a\}}=P_{\pi(A),\pi(B)}$ follows from Lemma \ref{lem:helper1}. 

Similarly, when $(A,B) \in \mathcal{D}_a^3$ then $(A\setminus \{a\},B ) \in \mathcal{T}_{a,b}$ and $P_{A,B} \geq P_{A\setminus \{a\}\cup \{b\},B}=P_{\pi(A),\pi(B)}$ follows from Lemma \ref{lem:helper1}.

We will now show that $a \succ^* b$ implies the existence of $(A,B) \in \mathcal{D}_a$ with $P_{A,B} > P_{\pi(A),\pi(B)}$. 

Assume not. Then in particular from (\ref{eq:PIsGeq}) we have that $P_{A,B} = P_{\pi(A),\pi(B)}$ holds for all $(A,B) \in \mathcal{D}_{a}$. 

\begin{claim*}
Let $\succ'$ be the relation defined by $A \succ' B$ iff $\pi(A) \succ \pi(B)$ with the corresponding winning probabilities defined by $P'_{A,B} = P_{\pi(A),\pi(B)}$. If $P_{A,B} = P_{\pi(A),\pi(B)}$ for every $(A,B) \in \mathcal{D}_{a}$ then $P'\in \mathbb{P}_{obs}$ and thus $\succ'\in C_{obs}$.
\end{claim*}

\begin{proof}
Observe that $P_{A,B} = P'_{A,B}$ for all disjoint teams $A$ and $B$ follows by definition. In addition, since $\pi$ is invertible and involuntary, for every team $A$ there exists a team $A_{\pi}$ such that $\pi(A_{\pi})=A$ hence $P_{A,A}=P_{A_{\pi},A_{\pi}}=1/2$. It remains to show that (1) Every pair of different teams $A,B$ holds $P'_{A,B} > 1/2$ iff $A \succ' B$, (2) that $\succ'$ is a total ordering satisfying single players consistency, and (3) that $P'$ satisfy SST w.r.t. $\succ'$. 

(1) %First,
Let $A,B$ be two different teams. 
It follows by the assumption over $P$ that $P'_{A,B} = P_{\pi(A),\pi(B)}  > 1/2$, iff $\pi(A) \succ \pi(B)$, which holds iff $A \succ' B$  by definition.

(2) We now show that  $\succ'$ is a strict total order.
From it's definition we have that  $\succ'$ is irreflexive. We also have that $\succ'$ is connected %\uskcom{Not sure it is clear to everyone what we mean by connected (can well be that I introduced it, sorry if that is the case)}
(and therefore strict) as $\pi$ is invertible and involutory, and every pair of different teams $A,B$ holds either $A\succ' B$ (if $\pi^{-1}(A)=\pi(A)\succ \pi(B)=\pi^{-1}(B)$) or $B\succ' A$ (if $\pi^{-1}(B)=\pi(B)\succ \pi(A)=\pi^{-1}(A)$), but not both. For transitivity, Consider a triplet of different teams, $A,B,C$ such that $A\succ' B\succ' C$ (and therefore $\pi^{-1}(A)=\pi(A)\succ \pi^{-1}(B)=\pi(B)\succ \pi^{-1}(C)=\pi(C)$).
From transitivity of $\succ$, we get $\pi^{-1}(A)=\pi(A)\succ \pi(C)=\pi^{-1}(C)$ which implies $A\succ' C$.
Hence the relation $\succ'$ is a strict total order.

We continue by showing that $\succ'$ satisfies single players consistency.\\
Let $x,y\in[n]$ be a pair of players and $S\in[n]\setminus \{x,y\}$ be a set of players such that $x\cup {S}\succ' y\cup {S}$.
We will show that $\{x\}\cup S'\succ' \{y\}\cup S'$ for all $S'\in[n]\setminus \{x,y\}$.%

Since $\pi$ %=\pi^{-1}$
is invertible, we know that there exist players $x_{\pi} = \pi(x)$ and $y_{\pi} = \pi(y)$, and a set, $S_{\pi} = \pi(S)\in [n]\setminus \{x_{\pi},y_{\pi}\}$,% and $S_{\pi}' = \pi(S')\in [n]\setminus \{x_{\pi},y_{\pi}\}$ . 
%Since $\pi$ %=\pi^{-1}$ is invertible, we know that there exists a set,  $S_{\pi}\in [n]\setminus \{x_{\pi},y_{\pi}\}$, 
such that 
$$\{x\}\cup S={\pi}^{-1}(\{x_{\pi}\}\cup S_{\pi})$$ and  $$\{y\}\cup S={\pi}^{-1}(\{y_{\pi}\}\cup S_{\pi}).$$
From the definition of $\succ'$, we get $$\{x_{\pi}\}\cup S_{\pi}\succ  \{y_{\pi}\}\cup S_{\pi}.$$
Therefore from the consistency of $\succ$ every $S_{\pi}'%={\pi}(\{x\}\cup S)\setminus \{x_{\pi}\}
\in [n]\setminus \{x_{\pi},y_{\pi}\}$ holds
$\{x_{\pi}\}\cup S_{\pi}'\succ  \{y_{\pi}\}\cup S_{\pi}'$ hence by definition $\{x\}\cup S'\succ'  \{y\}\cup S'$. 

% \uskcom{I have to go over this part again, but I believe it.}
(3) We now show that $P'$ satisfy SST w.r.t. $\succ'$.
% using the assumption that $P$ satisfy SST w.r.t. $\succ$.
Let $A\succ' B \succ' C$. From the definition of $\succ'$ we have that $\pi^{-1}(A)=\pi(A)\succ \pi^{-1}(B)=\pi(B) \succ \pi^{-1}(C)=\pi(C)$. As $P$ satisfy SST w.r.t. $\succ$,
\[
P_{\pi(A),\pi(C)}\geq \max\{P_{\pi(A),\pi(B)},P_{\pi(B),\pi(C)}\}
\]
Once again from the definition of $\succ'$,
\[
P_{A,C}'\geq \max\{P_{A,B}',P_{B,C}'\},
\]
Which means that $P'$ satisfy SST w.r.t. $\succ'$ by definition.
\end{proof}
%For the sake of illustration we present only one case, namely, that $(S,S') \in \mathcal{S}_{a,b}^*$ and both $(1)\;S \cup \{a\} \succ S' \cup \{b\}$ and $(2)\;S \cup \{b\} \succ S' \cup \{a\}$ holds. Assume for contradiction that $a\succ^* b$ does not hold. It thus follows that there exists an order, $\succ' \in \mathcal{C}_{obs}$ for which $b \succ' a$ holds. Let $P'$ be the corresponding probability matrix. Then, using consistency of $\succ'$ and $(1)$ respectively, we get $S \cup \{b\} \succ' S \cup \{a\} \succ' S' \cup \{b\}$ and from SST $P'_{S \cup \{b\},S' \cup \{b\}} \geq P'_{S \cup \{a\},S' \cup \{b\}}>1/2$. In addition, combining with $(2)$ it follows that $S \cup \{b\} \succ S' \cup \{b\} \succ S' \cup \{a\}$. Applying SST again we get $P'_{S \cup \{b\},S' \cup \{a\}} \geq P'_{S \cup \{b\},S' \cup \{b\}} \geq P'_{S \cup \{a\},S' \cup \{b\}}$, a contradiction to $(S,S') \in \mathcal{S}^*_{a,b}$. 
% \end{proof}

Now, observe that, together with the above claim, $a \succ^* b$ imply that for any $S \subseteq [n] \setminus \{a,b\}$ of size $k-1$ it holds that $\;S \cup \{a\} \succ^* S \cup \{b\}$ which implies $(i) \; S \cup \{a\} \succ S \cup \{b\}$ as well as $ (ii) \;S \cup \{a\} \succ' S \cup \{b\}$, as both $\succ$ and $\succ'$ are in $\mathcal{C}_{obs}$. Applying the definitions of $\succ'$ and $\pi$, statement $(ii)$ implies $\pi^{-1}(S \cup \{a\})=\pi(S \cup \{a\}) \succ \pi(S \cup \{b\})= \pi^{-1}(S \cup \{b\})$ which is equivalent to $S \cup \{b\} \succ S \cup \{a\}$ and hence yields a contradiction to $(i)$. 

We therefore deduce the existence of $(A,B)\in \mathcal{D}_{a}$ such that $P_{A,B}>P_{\pi(A),\pi(B)}$. From (\ref{eq:UnionD}), either 
 $(A,B) \in \mathcal{D}_a^2$, thus $(A\setminus \{a\},B \setminus \{b\}) \in \mathcal{S}_{a,b}$,  and $P_{A,B}> P_{A\setminus \{a\}\cup \{b\},B \setminus \{b\}\cup \{a\}}=P_{\pi(A),\pi(B)}$ yields $(A\setminus \{a\},B \setminus \{b\}) \in \mathcal{S}^*_{a,b}$, or $(A,B) \in \mathcal{D}_a^3$, thus $(A\setminus \{a\},B ) \in \mathcal{T}_{a,b}$ and $P_{A,B} > P_{A\setminus \{a\}\cup \{b\},B}=P_{\pi(A),\pi(B)}$ implies $(A\setminus \{a\},B ) \in \mathcal{T}^*_{a,b}$ 
 (As $(A,B) \in \mathcal{D}_a^1$, implies  $P_{\pi(A),\pi(B)}=P_{A,B}>P_{A,B}$ which is a contradiction.). 
 Overall, $\mathcal{S}^*_{a,b}\cup \mathcal{T}^*_{a,b}\ne \emptyset$.
\end{proof}

\section{Algorithms and Proofs of Section \ref{sec:stochastic}}\label{sec:stochasticFull}

\allowdisplaybreaks
We start by splitting the definition of $X_{a,b}(S,S',T)$ into two random variables, according to the two types of witnesses we introduced in the previous section. 
This will simplify the proof of Lemma \ref{lemma:SSTforF}.

For $(S,S') \in \mathcal{S}_{a,b}$ we introduce a random variable $Z_{a,b}(S,S')$ that combines the outcomes of the two duels obtained from the potential subsets witness $(S,S')$, namely $(S\cup\{a\}, S' \cup \{b\})$ and $(S'\cup\{a\}, S \cup \{b\})$ and similarly, a random variable $Y_{a,b}(S,T)$ that combines the outcomes of the two duels obtained by subset-team witness, $(S\cup\{a\},T)$ and $(T, S \cup \{b\})$.

\begin{definition}
For $a,b \in [n],a\neq b$, $(S,S') \in \mathcal{S}_{a,b}$ and $(S,T) \in \mathcal{T}_{a,b}$,
\begin{align*}
Z_{a,b}(S,S')&=\frac{\mathbbm{1}[
(S\cup\{a\}> S' \cup \{b\})]}{2}  +\frac{\mathbbm{1}[(S'\cup\{a\}> S \cup \{b\})]}{2},  \\
Y_{a,b}(S,T)&=
\frac{\mathbbm{1}[
\{a\}\cup S>T]}{2} +
\frac{\mathbbm{1}[
T>\{b\}\cup S]}{2}.
\end{align*}
\end{definition}
We note that both $Z_{a,b}(S,S')$ and $Y_{a,b}(S,T)$ can take values in $\{0,1/2,1\}$.

The random variables $Z_{a,b}$ and $Y_{a,b}$ are  the outcomes of picking random pairs, $(S,S') \in \mathcal{T}_{a,b}$ or $(S,T) \in \mathcal{S}_{a,b}$ and returning  $Z_{a,b}(S,S')$ and $Y_{a,b}(S,T)$, respectively. Observe that 
\begin{align*}
\E[Z_{a,b}]&=\sum_{(S,S') \in \mathcal{S}_{a,b}} \frac{\E[Z_{a,b}(S,S')]}{|\mathcal{S}_{a,b}|}
=\sum_{(S,S') \in \mathcal{S}_{a,b}} \frac{P_{\{a\}\cup S,\{b\}\cup S'}+P_{\{a\}\cup S,\{b\}\cup S'}}{2|\mathcal{S}_{a,b}|},\\
\E[Y_{a,b}]&= \sum_{(S,T) \in \mathcal{T}_{a,b}} \frac{\E[Y_{a,b}(S,T)]}{|\mathcal{T}_{a,b}|}
=\sum_{(S,T) \in \mathcal{T}_{a,b}} \frac{P_{\{a\}\cup S,T}+P_{T,\{b\}\cup S'}}{2|\mathcal{T}_{a,b}|},
\end{align*}
Where the expectation $\E[Z_{a,b}]$ is taken over all elements of $\mathcal{S}_{a,b}$ and the expectation $\E[Y_{a,b}]$ is taken over all elements $\mathcal{T}_{a,b}$. 

The  following lemma apply for every $a\succ b$, even if $a\nsucc^* b$. We prove Lemma using  SST and consistency.
\begin{restatable}{lemma}{lemmaSstForrSinglePlayers}\label{lemma:sstForrSinglePlayers}
Let $a,b \in [n]$ be any two players such that $a\succ b$. Then, 

(1) For every $(S,S')\in\mathcal{S}_{a,b}$ it holds that $\E[Z_{a,b}(S,S')]\geq 1/2$. 

(2) For every $(S,T)\in\mathcal{T}_{a,b}$ it holds that $\E[Y_{a,b}(S,T)]\geq 1/2$.
\end{restatable}

\begin{proof}
(1) 
Let $(S,S') \in \mathcal{S}_{a,b}$ and $a\succ b$. Then, 
\begin{align*}
&\E[Z_{a,b}(S,S')]=\frac{P_{\{a\}\cup S,\{b\}\cup S'}+P_{\{a\}\cup S',\{b\}\cup S}}{2}\geq \frac{1}{2} \\ \iff&
P_{\{a\}\cup S,\{b\}\cup S'}+P_{\{a\}\cup S',\{b\}\cup S}\geq 1 \\ \iff&
P_{\{a\}\cup S,\{b\}\cup S'}\geq 1-P_{\{a\}\cup S',\{b\}\cup S} \\ \iff & 
P_{\{a\}\cup S,\{b\}\cup S'}\geq P_{\{b\}\cup S,\{a\}\cup S'},
\end{align*}
which holds according to Lemma \ref{lem:helper1}.\\
(2)
Let $(S,T) \in \mathcal{T}_{a,b}$ and $a\succ b$. From Lemma \ref{lem:helper1} we have that 
\[
P_{\{a\}\cup S,T}\geq P_{\{b\}\cup S,T},
\]
which is equivalent to 
\[
P_{\{a\}\cup S,T}\geq P_{\{b\}\cup S,T}=1-P_{T,\{b\}\cup S}
\]
and therefore 
\[
2\E[Y_{a,b}(S,T)]
\geq 1.
\]
Hence, $\E[Y_{a,b}(S,T)]
\geq 1/2$.
\end{proof}

\begin{corollary}\label{cor:relationImpliesYandZGeqHalf}
For players $a,b\in[n]$ such that $a\succ  b$ then $\E[Z_{a,b}],\E[Y_{a,b}]\geq 1/2$.
\end{corollary}

For the definition of $X_{a,b}(S,S',T)$ we refer to the main part of our paper. 
In the following we show how $X_{a,b}(S,S',T)$ can be expressed by $Z_{a,b}(S,S')$ and $Y_{a,b}(S,T)$, namely
\begin{align*}
    X_{a,b}(S,S',T) &= \frac{\mathbbm{1}[S \cup \{a\} > S' \cup \{b\}] - \mathbbm{1}[S \cup \{b\} > S' \cup \{a\}]}{4}  + \frac{\mathbbm{1}[S \cup \{a\} > T] - \mathbbm{1}[S \cup \{b\} > T]}{4} \\ 
    & = \frac{\mathbbm{1}[S \cup \{a\} > S' \cup \{b\}] + \mathbbm{1}[S' \cup \{a\} > S \cup \{b\}] - 1  }{4}\\ &
    + \frac{ \mathbbm{1}[S \cup \{a\} > T] + \mathbbm{1}[T > S \cup \{b\}] - 1}{4}\\
    &= \frac{Z_{a,b}(S,S') + Y_{a,b}(S,T) - 1}{2}.
\end{align*}

In similar fashion to the definitions of $\mathcal{S}_{a,b}$, $\mathcal{S}^*_{a,b}$ and $Z_{a,b}$ w.r.t. $Z(S,S')$, we defined
$$\mathcal{X}_{a,b}=\{(S,S',T)| (S,S')\in \mathcal{S}_{a,b}, (S,T)\in \mathcal{T}_{a,b}\},$$
and the random variable $X_{a,b}$ to be the outcome of picking a random triplet, $(S,S',T) \in \mathcal{X}_{a,b}$ and returning $X_{a,b}(S,S',T)$. 

The set $\mathcal{X}_{a,b}^*$ contains all triplets $(S,S',T) \in \mathcal{X}_{a,b}$ such that either $(S,S')\in\mathcal{S}^*_{a,b}$ or $(S,T)\in\mathcal{T}^*_{a,b}$.
Note that the support of each $X_{a,b}(S,S',T)$ is included  $\{-1/2,-1/4,0,1/4,1/2\}$ and that $\E[X_{a,b}]\in [-1/2,1/2]$.

For the next Theorem's proof we rely on Theorem \ref{thm:provableSingleRelationIffWitnessExists}, Corollary \ref{cor:relationImpliesYandZGeqHalf} in one direction, and show the other using the probabilistic method.
\thmTheOneTrueLemmaToRuleAllWitnesses*
\begin{proof}
We will show that for players $a,b\in[n]$ it holds that $a\succ^*  b$ iff one of the following holds:\\
(1) $\E[Z_{a,b}]>1/2$, or \\
(2) $\E[Y_{a,b}]>1/2$.\\
This is equivalent to $\E[X_{a,b}]>0$ according to the definition of $X_{a,b}$ and Corollary \ref{cor:relationImpliesYandZGeqHalf}.\\
$(\Rightarrow)$ If $a\succ^* b$ then from Theorem \ref{thm:provableSingleRelationIffWitnessExists} we know that one of the following holds:
\begin{enumerate}
    \item There exists a subsets, witness 
$(S,S')\in\mathcal{S}_{a,b}$ for $a\succ b$. So by definition  $\E[Z_{a,b}(S,S')]>1/2$, and combined with Lemma \ref{lemma:sstForrSinglePlayers}
we have $\E[Z_{a,b}]>1/2$.
\item There exists a subset-team witness $(S,T)\in \mathcal{T}_{a,b}$ for $a\succ b$. Thus  $\E[Y_{a,b}(S,T)]>1/2$, hence Lemma \ref{lemma:sstForrSinglePlayers} implies that $\E[Y_{a,b}]>1/2$.
\end{enumerate}
$(\Leftarrow)$ If (1) holds, the probabilistic method implies the existence of $(S,S')\in\mathcal{S}_{a,b}$ such that $\E[Z_{a,b}(S,S')]>1/2$ which means that $(S,S')$ is a witness for $a\succ b$, hence, $a \succ^* b$ by Theorem \ref{thm:provableSingleRelationIffWitnessExists}.
 If (2) holds, the probabilistic method implies that there exists $(S,T)\in\mathcal{T}_{a,b}$ such that $\E[Y_{a,b}(S,T)]>1/2$ which means that $(S,T)$ is a witness for $a\succ b$, hence, $a \succ^* b$ by Theorem \ref{thm:provableSingleRelationIffWitnessExists}.
 
 Thus according to the definition of $X_{a,b}$ the theorem holds.
\end{proof}

\paragraph{Gap parameter}
Recall that we defined our gap parameter by $\Delta = \E[X_{k,k+1}]$. In the following we show that our gap parameter does not just 
help us to distinguish between the top $k$ and the top $k+1$ players, but also between other players in $A^*_k$ and players from $[n] \setminus A^*_k$. 
To this end, we show in Lemma \ref{lemma:SSTforF} that strong stochastic transitivity holds for $\E[X_{a,b}]$. 
For most elements $(S,S',T) \in \mathcal{X}_{a,b}$ it holds that $\E[X_{a,c}(\pi(S),\pi(S'),\pi(T))] \geq \E[X_{a,b}(S,S',T)]$ (and analogously for $X_{b,c}$), where $\pi$ is a permutation exchanging $b$ and $c$, but, surprisingly, this is not true in general. 
By constructing a charging scheme, we can still show that this holds in expectation over all elements of $\mathcal{X}_{a,b}$, and derive a strong stochastic transitivity for distinguishabilities w.r.t. the total order $\succ$ on the players.

The proof of the following lemma also shows that from every $a\succ b$ witness $(S,S',T)\in \mathcal{X}^*_{a,b}$, and for any player $c$ such that $b\succ c$ we can create a $a\succ c$- witness. Similarly, from every $b\succ c$ witness $(S,S',T)\in \mathcal{X}^*_{b,c}$, and for any player $a$ such that $a\succ b$ we can create a $a\succ c$- witness. 

\lemmaSSTforF*
\begin{proof}

In the following we show that $\E[X_{a,c}] \geq \E[X_{a,b}]$. The proof that
$\E[X_{a,c}] \geq \E[X_{b,c}]$ works completely analogously and is therefore omitted. Let $\pi$ be the function exchanging $b$ and $c$, i.e. \[\pi(S) = \begin{cases} S \setminus \{c\} \cup \{b\} & \text{ if }c \in S, b \not\in S  \\ S \setminus \{b\} \cup \{c\} & \text{ if } b \in S, c \not\in S \\ S & \text{ else.}\end{cases}\]
Then, we define the function $f: \mathcal{X}_{a,b} \rightarrow \mathcal{X}_{a,c}$ by $f(S,S',T) = (\pi(S),\pi(S'),\pi(T))$. Observe that, for this application of $\pi$, the second case within the definition of $\pi$ never occurs, as none of the sets $S,S',T$ contains $b$ when $(S,S',T) \in \mathcal{X}_{a,b}$.
It will we helpful to partition $\mathcal{X}_{a,b}$ in the following way. 
\begin{align*}
    \mathcal{X}_{a,b}^{1} &= \{(S,S',T) \in \mathcal{X}_{a,b} \mid c \not\in S \cup S' \cup T\} \\ 
    \mathcal{X}_{a,b}^{2} &= \{(S,S',T) \in \mathcal{X}_{a,b} \mid c \in S\}\\
    \mathcal{X}_{a,b}^{3} &= \{(S,S',T) \in \mathcal{X}_{a,b} \mid c \in S' \setminus T\}\\
    \mathcal{X}_{a,b}^{4} &= \{(S,S',T) \in \mathcal{X}_{a,b} \mid c \in T \setminus S'\}\\
    \mathcal{X}_{a,b}^{5} &= \{(S,S',T) \in \mathcal{X}_{a,b} \mid c \in T \cap S'\}.
\end{align*}

Then we can also define $\mathcal{X}^{i}_{a,c} = \{f(S,S',T) \mid (S,S',T) \in \mathcal{X}_{a,b}^i\}$ for all $i \in \{1, \dots, 5\}$. Observe that $\{\mathcal{X}^i_{a,c} \mid i \in \{1,\dots, 5\}\}$ is also a partition of $\mathcal{X}_{a,c}$.  

We will start by proving that for every $(S,S',T)\in \mathcal{X}^1_{a,b} \cup \mathcal{X}^2_{a,b} \cup  \mathcal{X}^3_{a,b} \cup \mathcal{X}^4_{a,b} \cup  \mathcal{X}^5_{a,b}$ 

\begin{equation}\label{eq:strongguaranteeZ}
\E[{Z}_{a,c}(f(S,S'))]\geq \E[Z_{a,b}(S,S')]
\end{equation}
and for all $(S,S',T) \in \mathcal{X}^1_{a,b} \cup \mathcal{X}^2_{a,b} \cup \mathcal{X}^3_{a,b}$
\begin{equation}\label{eq:strongguaranteeY}
\E[{Y}_{a,c}(f(S,T))]\geq \E[Y_{a,b}(S,T)]
\end{equation}

by exhaustion.

\begin{enumerate}[label=(\roman*)]
    \item Let $(S,S',T) \in \mathcal{X}_{a,b}^1$. We get that $f(S,S',T)=(S,S',T)$ and both
    \begin{align*}
    \E[Z_{a,c}(S,S')]&=\frac{P_{\{a\}\cup S,\{c\}\cup S'}+P_{\{a\}\cup S',\{c\}\cup S}}{2} \geq
    \frac{P_{\{a\}\cup S,\{b\}\cup S'}+P_{\{a\}\cup S',\{b\}\cup S}}{2}\\&=\E[Z_{a,b}(S,S')],
    \end{align*}
    \begin{align*}
    \E[Y_{a,c}(S,T)]&=\frac{P_{\{a\}\cup S,T}+P_{T,\{c\}\cup S}}{2} \geq \frac{P_{\{a\}\cup S,T}+P_{T,\{b\}\cup S}}{2}\\
    &=\E[Y_{a,b}(S,T)]
    \end{align*}
    follow from consistency and SST. 
    \item Let $(S,S',T) \in \mathcal{X}_{a,b}^2$. Then, $f(S,S',T)=(S\setminus \{c\}\cup \{b\},S,T)$ and both
    \begin{align*}
    \E[Z_{a,c}(S\setminus \{c\}\cup \{b\},S')]&=\frac{P_{\{a\}\cup S\setminus \{c\}\cup \{b\},\{c\}\cup S'}+P_{\{a\}\cup S',\{c\}\cup S\setminus \{c\}\cup \{b\}}}{2}\\ &\geq
    \frac{P_{\{a\}\cup S,\{b\}\cup S'}+P_{\{a\}\cup S',\{b\}\cup S}}{2}\\&=\E[Z_{a,b}(S,S')]
    \end{align*}
    \begin{align*}
    \E[Y_{a,c}(S\setminus \{c\}\cup \{b\},T)])& =\frac{P_{\{a\}\cup S\setminus \{c\}\cup \{b\},T}+P_{T,\{c\}\cup S}}{2}\\
    &\geq 
    \frac{P_{\{a\}\cup S,T}+P_{T,\{b\}\cup S}}{2}
    \\&=\E[Y_{a,b}(S,T)]
    \end{align*}
    follow as $\{c\}\cup S\setminus \{c\}\cup \{b\}=S\cup \{b\}$ and from consistency and SST yield the rest.
\item Let $(S,S',T) \in \mathcal{X}_{a,b}^3$. Then,  $f(S,S',T)=(S,S'\setminus \{c\}\cup \{b\},T)$ and 
    \begin{align*}
    \E[Z_{a,c}(S,S'\setminus \{c\}\cup \{b\}))]& =\frac{P_{\{a\}\cup S,\{c\}\cup S'\setminus \{c\}\cup \{b\}}+P_{\{a\}\cup S'\setminus \{c\}\cup \{b\},\{c\}\cup S}}{2}\\ & \geq
    \frac{P_{\{a\}\cup S,\{b\}\cup S'}+P_{\{a\}\cup S',\{b\}\cup S}}{2}\\ & =\E[Z_{a,b}(S,S')]
    \end{align*}
    follows as $\{c\}\cup S'\setminus \{c\}\cup \{b\}=S'\cup \{b\}$ and  consistency and SST yield the rest. In addition, we already showed that in this case  thus $\E[{Y}_{a,c}(S,T)]\geq \E[Y_{a,b}(S,T)]$ (due to the same reason as in (i)). 
\item Let $(S,S',T) \in \mathcal{X}_{a,b}^4$. Then, $f(S,S',T)=(S,S',T\setminus \{c\}\cup \{b\})$. Observe that we have already shown that $\E[Z_{a,c}(S,S')] \geq \E[{Z_{a,b}(S,S')}]$ in this case (due to the same reason as (i)). 
\item Let $(S,S',T) \in \mathcal{X}_{a,b}^5$. Then, $f(S,S',T)=(S,S'\setminus \{c\} \cup \{b\},T\setminus \{c\}\cup \{b\})$. Observe that we have already shown that $\E[Z_{a,c}(S,S'\setminus \{c\} \cup \{b\})] \geq \E[{Z_{a,b}(S,S')}]$ in this case (due to the same reason as (iii)). 
\end{enumerate}

This concludes the proof of equations (\ref{eq:strongguaranteeZ}) and (\ref{eq:strongguaranteeY}). In particular, from (ii) and (iii) it directly follows that  \begin{equation}\label{eq:boundforXi}
    \sum_{(S,S',T) \in \mathcal{X}_{a,c}^i} \E [X(S,S',T)] = \sum_{(S,S',T) \in \mathcal{X}_{a,b}^i} \E [X(f(S,S',T))] \geq \sum_{(S,S',T) \in \mathcal{X}_{a,b}^i} \E [X(S,S',T)] 
\end{equation}
holds for $i \in \{2,3\}$. 

We will continue the proof by showing that, for every $(S,T) \in \mathcal{S}_{a,b}$ with $c \in T$, it holds that  
\begin{align}\label{eq:charging}
     \E[Z_{a,c}(S,T \setminus \{c\})] + \E[Y_{a,c}(S,T \setminus \{c\} \cup \{b\})] \geq\E[Z_{a,b}(S,T\setminus \{c\})] + \E[Y_{a,b}(S,T)].
\end{align}
This will then be helpful to conclude the proof. 

To this end, observe that 
\begin{align*}
     &\E[Z_{a,c}(S,T \setminus \{c\})] + \E[Y_{a,c}(S,T \setminus \{c\} \cup \{b\})]) \\ 
    &= P_{S \cup \{a\},T} + P_{T \setminus \{c\} \cup \{a\},S \cup \{c\}} + P_{S \cup \{a\},T \setminus \{c\} \cup \{b\}} + P_{T \setminus \{c\} \cup \{b\},S \cup \{c\}}\\ 
    &=P_{S \cup \{a\},T \setminus \{c\} \cup \{b\}}  + P_{T \setminus \{c\} \cup \{a\},S \cup \{c\}} + P_{S \cup \{a\},T} +  P_{T \setminus \{c\} \cup \{b\},S \cup \{c\}}\\ 
    &\geq P_{S \cup \{a\},T \setminus \{c\} \cup \{b\}} + P_{T \setminus \{c\} \cup \{a\},S \cup \{b\}} + P_{S \cup \{a\},T} + P_{T,S \cup \{b\}}\\
    &=\E[Z_{a,b}(S,T\setminus \{c\})] + \E[Y_{a,b}(S,T)], 
\end{align*}

which follows by consistency and SST. This will now be helpful to establish a charging scheme. Namely, we are first going to show that \begin{equation}\label{eq:boundforX4} \sum_{(S,S',T) \in \mathcal{X}_{a,c}^4} \E[X_{a,c}(S,S',T)] = \sum_{(S,S',T) \in \mathcal{X}_{a,b}^4} \E [X_{a,c}(f(S,S',T))]\geq \sum_{(S,S',T) \in \mathcal{X}_{a,b}^4} \E[X_{a,b}(S,S',T)].\end{equation}

This is true since

\begin{align*}
&\sum_{(S,S',T) \in \mathcal{X}_{a,c}^4} 2 \E[X_{a,c}(S,S',T)]  + |\mathcal{X}_{a,c}|\\
    &\sum_{(S,S',T) \in \mathcal{X}_{a,b}^4}  2  \E[X_{a,c}(S,S',T \setminus \{c\} \cup \{b\})] + |\mathcal{X}_{a,c}|\\& = \sum_{(S,S',T) \in \mathcal{X}_{a,b}^4} (\E[Z_{a,c}(S,S')] + \E[Y_{a,c}(S,T \setminus \{c\} \cup \{b\})]) \\ 
    &= {n-k-2 \choose k-1} \Big( \sum_{(S,S') \in \mathcal{S}_{a,b} \cap \mathcal{S}_{a,c}} \E[Z_{a,c}(S,S')] +  \sum_{(S,T) \in \mathcal{T}_{a,b} \mid c \in T} \E[Y_{a,c}(S,T \setminus \{c\} \cup \{b\})] \Big)\\
    &= {n-k-2 \choose k-1} \Big( \sum_{(S,S') \in \mathcal{S}_{a,b} \cap \mathcal{S}_{a,c}} \E[Z_{a,c}(S,S')] +  \sum_{(S,S') \in \mathcal{S}_{a,b} \cap \mathcal{S}_{a,c}} \E[Y_{a,c}(S,S' \cup \{b\})] \Big)\\
    &= {n-k-2 \choose k-1} \Big( \sum_{(S,S') \in \mathcal{S}_{a,b} \cap \mathcal{S}_{a,c}} \E[Z_{a,c}(S,S')] + \E[Y_{a,c}(S,S' \cup \{b\})] \Big) \\ 
    &\geq {n-k-2 \choose k-1} \Big( \sum_{(S,S') \in \mathcal{S}_{a,b} \cap \mathcal{S}_{a,c}} \E[Z_{a,b}(S,S')] + \E[Y_{a,b}(S,S' \cup \{c\})] \Big) \\
    &= {n-k-2 \choose k-1} \Big( \sum_{(S,S') \in \mathcal{S}_{a,b} \cap \mathcal{S}_{a,c}} \E[Z_{a,b}(S,S')] +  \sum_{(S,S') \in \mathcal{S}_{a,b} \cap \mathcal{S}_{a,c}} \E[Y_{a,b}(S,S' \cup \{c\})] \Big)\\
    &= {n-k-2 \choose k-1} \Big( \sum_{(S,S') \in \mathcal{S}_{a,b} \cap \mathcal{S}_{a,c}} \E[Z_{a,b}(S,S')] +  \sum_{(S,T) \in \mathcal{T}_{a,b} \mid c \in T} \E[Y_{a,b}(S,T)] \Big)\\
    &=\sum_{(S,S',T) \in \mathcal{X}_{a,b}^4} (\E[Z_{a,b}(S,S')] + \E[Y_{a,b}(S,T)]) \\ 
    &\sum_{(S,S',T) \in \mathcal{X}_{a,b}^4} 2\E[X_{a,b}(S,S',T)] + |\mathcal{X}_{a,b}|,
\end{align*}
where the inequality follows by equation (\ref{eq:charging}). This completes the proof of (\ref{eq:boundforX4}).

Next, we are going to show that a similar bound holds when we sum over elements in $\mathcal{X}_{a,b}^1 \cup \mathcal{X}_{a,b}^5$. More precisely, we are going to show that \begin{align}\label{eq:boundforX5} \sum_{(S,S',T) \in \mathcal{X}_{a,c}^1 \cup \mathcal{X}_{a,c}^5} \E[X_{a,c}(S,S',T)] &= \sum_{(S,S',T) \in \mathcal{X}_{a,b}^1 \cup \mathcal{X}_{a,b}^5} \E [X_{a,c}(f(S,S',T))] \nonumber \\ &\geq \sum_{(S,S',T) \in \mathcal{X}_{a,b}^1 \cup \mathcal{X}_{a,b}^5} \E[X_{a,b}(S,S',T)].\end{align}

To this end, observe that 

\begin{align*}
 &\sum_{(S,S',T) \in \mathcal{X}_{a,c}^1} {2\E[X_{a,c}(S,S',T)]} + \sum_{(S,S',T) \in \mathcal{X}_{a,c}^5}{2\E[X_{a,c}(S,S',T)]} + |\mathcal{X}_{a,c}^1| + |\mathcal{X}_{a,c}^5|\\ 
    &\sum_{(S,S',T) \in \mathcal{X}_{a,b}^1} 2\E[X_{a,c}(S,S',T)] + \sum_{(S,S',T) \in \mathcal{X}_{a,b}^5} 2\E[X_{a,c}(S,S' \setminus \{c\} \cup \{b\},T \setminus \{c\} \cup \{b\})] + |\mathcal{X}_{a,b}^1| + |\mathcal{X}_{a,b}^5|\\ 
    &= { n-k-2\choose k}\sum_{(S,S') \in \mathcal{S}_{a,b} \cap \mathcal{S}_{a,c}} \E[Z_{a,c}(S,S')] + {n-k-3 \choose k-1}\sum_{(S,T) \in \mathcal{T}_{a,b} \cap \mathcal{T}_{a,c}}\E[Y_{a,c}(S,T)] \\ &+ {n-k-2\choose k-1}  \sum_{(S,S') \in \mathcal{S}_{a,b}\mid c \in S'}\E[Z_{a,c}(S,S' \setminus \{c\} \cup \{b\})] + {n-k-2 \choose k-2}\sum_{(S,T) \in \mathcal{T}_{a,b} \mid c \in T}\E[Y_{a,c}(S,T \setminus \{c\} \cup \{b\})] \\
    &= { n-k-2\choose k}\sum_{(S,S') \in \mathcal{S}_{a,b} \cap \mathcal{S}_{a,c}} \E[Z_{a,c}(S,S')] +  [\dots] + {n-k-2 \choose k-2}\sum_{(S,T) \in \mathcal{T}_{a,b} \mid c \in T}\E[Y_{a,c}(S,T \setminus \{c\} \cup \{b\})] \\
    &= \Big({ n-k-2\choose k} - { n-k-2\choose k-2}\Big)\sum_{(S,S') \in \mathcal{S}_{a,b} \cap \mathcal{S}_{a,c}} \E[Z_{a,c}(S,S')] +  [\dots] \\
    &+ {n-k-2 \choose k-2}\sum_{(S,S') \in \mathcal{S}_{a,b} \cap \mathcal{S}_{a,c}}\E[Z_{a,c}(S,S')] + \E[Y_{a,c}(S,S' \cup \{b\})] \\
    & \geq \Big({ n-k-2\choose k} - { n-k-2\choose k-2}\Big)\sum_{(S,S') \in \mathcal{S}_{a,b} \cap \mathcal{S}_{a,c}} \E[Z_{a,b}(S,S')] +  [\dots] \\
    &+ {n-k-2 \choose k-2}\sum_{(S,S') \in \mathcal{S}_{a,b} \cap \mathcal{S}_{a,c}}\E[Z_{a,b}(S,S')] + \E[Y_{a,b}(S,S' \cup \{c\})] \\
    & = { n-k-2\choose k}\sum_{(S,S') \in \mathcal{S}_{a,b} \cap \mathcal{S}_{a,c}} \E[Z_{a,b}(S,S')] +  [\dots] + {n-k-2 \choose k-2}\sum_{(S,S') \in \mathcal{S}_{a,b} \cap \mathcal{S}_{a,c}} \E[Y_{a,b}(S,S' \cup \{c\})] \\
    & = { n-k-2\choose k}\sum_{(S,S') \in \mathcal{S}_{a,b} \cap \mathcal{S}_{a,c}} \E[Z_{a,b}(S,S')] +  [\dots] + {n-k-2 \choose k-2}\sum_{(S,S') \mathcal{T}_{a,b} \mid c \in T}\E[Y_{a,b}(S,T)] \\
     &= { n-k-2\choose k}\sum_{(S,S') \in \mathcal{S}_{a,b} \cap \mathcal{S}_{a,c}} \E[Z_{a,b}(S,S')] + {n-k-3 \choose k-1}\sum_{(S,T) \in \mathcal{T}_{a,b} \cap \mathcal{T}_{a,c}}\E[Y_{a,c}(S,T)] \\ &+ {n-k-2\choose k-1}  \sum_{(S,S') \in \mathcal{S}_{a,b}\mid c \in S'}\E[Z_{a,c}(S,S' \setminus \{c\} \cup \{b\})] + {n-k-2 \choose k-2}\sum_{(S,T) \in \mathcal{T}_{a,b} \mid c \in T}\E[Y_{a,b}(S,T)] \\
     &\geq { n-k-2\choose k}\sum_{(S,S') \in \mathcal{S}_{a,b} \cap \mathcal{S}_{a,c}} \E[Z_{a,b}(S,S')] + {n-k-3 \choose k-1}\sum_{(S,T) \in \mathcal{T}_{a,b} \cap \mathcal{T}_{a,c}}\E[Y_{a,b}(S,T)] \\ &+ {n-k-2\choose k-1}  \sum_{(S,S') \in \mathcal{S}_{a,b}\mid c \in S'}\E[Z_{a,b}(S,S')] + {n-k-2 \choose k-2}\sum_{(S,T) \in \mathcal{T}_{a,b} \mid c \in T}\E[Y_{a,b}(S,T)] \\
     &=\sum_{(S,S',T) \in \mathcal{X}_{a,b}^1} 2\E[X_{a,b}(S,S',T)] + \sum_{(S,S',T) \in \mathcal{X}_{a,b}^5} 2\E[X_{a,b}(S,S',T)] + |\mathcal{X}_{a,c}^1| + |\mathcal{X}_{a,c}^5|,\\ 
\end{align*}
where the first inequality follows by equation (\ref{eq:charging}) and (\ref{eq:strongguaranteeZ}) and the second inequality follows from equation (\ref{eq:strongguaranteeZ}) and (\ref{eq:strongguaranteeY}). The dots ($[\dots]$) stands for \[{n-k-3 \choose k-1}\sum_{(S,T) \in \mathcal{T}_{a,b} \cap \mathcal{T}_{a,c}}\E[Y_{a,c}(S,T)] + {n-k-2\choose k-1}  \sum_{(S,S') \in \mathcal{S}_{a,b}\mid c \in S'}\E[Z_{a,c}(S,S' \setminus \{c\} \cup \{b\})], \] which is a part of the expression that it is omitted during the calculations for the sake of brevity.
Summarizing, we get that
\begin{align*}
    \E[X_{a,c}] &= \frac{\sum_{(S,S',T) \in \mathcal{X}_{a,c}} \E[X_{a,c}(S,S',T)]}{|\mathcal{X}_{a,c}|} = \frac{\sum_{i=1}^5\sum_{(S,S',T) \in \mathcal{X}_{a,c}^i} \E[X_{a,c}(S,S',T)]}{|\mathcal{X}_{a,c}|} \\ 
    & \geq \frac{\sum_{i=1}^5\sum_{(S,S',T) \in \mathcal{X}_{a,b}^i} \E[X_{a,b}(S,S',T)]}{|\mathcal{X}_{a,c}|} =\E[X_{a,b}],  
\end{align*}
where the inequality follows from equations (\ref{eq:boundforXi}), (\ref{eq:boundforX4}), and (\ref{eq:boundforX5}). The last inequality follows from $|\mathcal{X}_{a,b}| = |\mathcal{X}_{a,c}|$.
\end{proof}

\paragraph{The reduction} We close this section by giving the two subroutines mentioned within the reduction to the classic dueling bandits setting. 

\begin{algorithm}[h]
\begin{algorithmic}
\STATE \textbf{Input:} Players  $a,b\in [n]$
\STATE \textbf{Output:} $w\in\{0,1\}$ such that $w=1$ if $a$ won and $w=0$ if $b$ won. 
\STATE Pick $(S,S',T)\in \mathcal{X}_{a,b}$ randomly
\STATE $z\leftarrow (\mathbbm{1}[\{a\}\cup S > \{b\}\cup S']+\mathbbm{1}[\{a\}\cup S' > \{b\}\cup S])/2$
\STATE $y\leftarrow( \mathbbm{1}[\{a\}\cup S > T]+\mathbbm{1}[T > \{b\}\cup S])/2$
\STATE $x\leftarrow (z+ y-1)/2$ 
\STATE \textbf{return} sample of a biased coin with bias $1/2+x$
\end{algorithmic}\caption{singlesDuel: simulation of a duel between single players} \label{alg:singlesDuel}
\end{algorithm}

\section{Algorithms and Proofs of Section \ref{sec:deterministic}}\label{sec:deterministicFull}

\medskip
\paragraph{Uncover Subroutine}
As sketched within the main part of our paper, we refine the idea of the \emph{Uncover} subroutine by a binary search approach. Moreover, we add the option to input a refinement of $A$ and $B$, namely $A = A^{(1)} \cup A^{(2)}$, $B=B^{(1)} \cup B^{(2)}$, guaranteeing that the uncovered relation is between a pair of players from $A^{(1)}$ and $B^{(1)}$, while $A^{(2)}$ and $B^{(2)}$ are contained in one of the sets of the witness each. For that to work, we require that 
\begin{enumerate}[label=(\alph*)]
    \item $|A^{(1)}| + |A^{(2)}| = k$, 
    \item $|A^{(i)}| = |B^{(i)}|$ for $i \in \{1,2\}$, 
    \item $A^{(1)} \cup A^{(2)} \succ B^{(1)} \cup B^{(2)}$, and 
    \item $A^{(1)} \cup B^{(2)} \succ B^{(1)} \cup A^{(2)}$.
\end{enumerate}
   Observe that for any four sets satisfying $(a)$ and $(b)$ one of the four sets wins in both duels. By enforcing $(c)$ and $(d)$ we fix wlog that this set is $A^{(1)}$. Let us assume that the sets $A^{(1)}$ and $B^{(1)}$ are ordered, meaning that $A^{(1)} = \{a_1, \dots, a_{|A^{(1)}|}\}$ and $B^{(1)} = \{b_1, \dots, b_{|A^{(1)}|}\}$. We also introduce the shorthand notation $A_{\ell:r}$ for $\{a_{\ell},\dots, a_{r}\}$ and respectively $B_{\ell:r}$ for $\{b_{\ell},\dots, b_{r}\}$ for any $\ell, r \in [|A^{(1)}|]$. The subroutine is formalized in Algorithm \ref{sub:Uncover}.

\begin{algorithm}[H]
\begin{algorithmic}
\STATE \textbf{Input: } four disjoint sets, $A^{(1)},B^{(1)},A^{(2)},B^{(2)}$ with $|A^{(1)}|=|B^{(1)}|,|A^{(2)}|=|B^{(2)}|$, $|A^{(1)}| + |A^{(2)}| = k$, $A^{(1)} \cup A^{(2)} \succ B^{(1)} \cup B^{(2)}$, and $ A^{(1)} \cup B^{(2)} \succ B^{(1)} \cup A^{(2)}$\\
\STATE \textbf{Output: } $a \in A^{(1)}, b \in B^{(1)},$ $(S,S') \in \mathcal{S}^*_{a,b}$ with ($C \subseteq S$ and $D \subseteq S'$) or ($D \subseteq S$ and $C \subseteq S'$) \\
\smallskip
\STATE Set  $S \leftarrow A^{(1)} \cup A^{(2)}$, $T \leftarrow B^{(1)} \cup B^{(2)}$, $\ell \leftarrow 1$, $r \leftarrow |A^{(1)}|$ \\ 
\WHILE{$\ell < r$ }
\STATE $i\leftarrow \big\lfloor\frac{\ell+r}{2}\big\rfloor$\\
\STATE $S\leftarrow S-A_{i+1:r}\cup B_{i+1:r} $\\
\STATE $T\leftarrow T-B_{i+1:r}\cup A_{i+1:r} $\\
\IF{$S \succ T$}
\STATE $r\leftarrow i$
\ELSE \STATE $\ell\leftarrow i+1$ \\ 
\STATE swap $S$ and $T$
\ENDIF
\ENDWHILE
\STATE \textbf{return} $(a_{\ell},b_{\ell}),$ and $(S \setminus \{a_{\ell}\},T \setminus \{b_{\ell}\})$
\end{algorithmic}
\caption{Uncover Subroutine}\label{sub:Uncover}
\end{algorithm}

In order to show that the algorithm is well-defined and works correctly, the following Lemma will be helpful.

\begin{lemma}\label{lem:unc-log}
In subroutine Uncover (Algorithm \ref{sub:Uncover}),
at the end of every while loop, it holds that, $(i) \;\ell, r \in \mathbb{N}$ with $\ell \leq r$, $(ii) \;A_{\ell:r} \subseteq S$, $B_{\ell:r} \subseteq T$, $(iii)\; S\succ T$, and $ (iv)\;T \setminus B_{\ell:r} \cup A_{\ell:r} \succ S \setminus A_{\ell:r} \cup B_{\ell:r}$, $(v)$ exactly one of $S$ and $T$ contains $A^{(2)}$, the other set contains $B^{(2)}$. 
\end{lemma}

\begin{proof}
We prove all statements via one joint induction over the iterations of the while loop. All statements are clearly true at the beginning of the first while loop. Now, consider any iteration in which the four statements are true at the beginning of the while loop. It suffices to show that they are still true after resetting $S$, $T$, $\ell$, and $r$. For clarity, we refer to the modified variables of the teams just before the if condition as $S'$, $T'$ and after the if condition as $S''$, $T''$. Similarly, $\ell'$, and $r'$ are the values of the indices after the if condition. In the following, we show that the four conditions still hold for $S''$,$T''$,$\ell'$, and $r'$.

\textbf{Case 1: }$S' \succ T'$. Then, $S'' = S'$, $T'' = T'$, $\ell' = \ell$, $r'=i$. The condition of the while loop, $\ell <r$, clearly implies that $\ell'= \ell \leq \lfloor \frac{\ell+r}{2} \rfloor = i =r'$. Moreover, by construction $A_{\ell:i} = A_{\ell':r'} \subseteq S''$ and  $B_{\ell:i} = B_{\ell':r'} \subseteq T''$ and hence condition $(ii)$ is satisfied. Condition $(iii)$, i.e., $S'' \succ T''$ is satisfied by the case condition. For condition $(iv)$ let us rewrite the induction hypothesis for condition $(iv)$ as \[T - B_{\ell:r} \cup (A_{\ell:i} \cup A_{i+1:r}) \succ S - A_{\ell:r} \cup (B_{\ell:i} \cup B_{i+1:r}). \] Observe that $T - B_{\ell:r} \cup A_{i+1:r} = T' - B_{\ell:i} $ and $S - A_{\ell:r} \cup B_{i+1:r} = S' - A_{\ell:i} $. Hence, the above expression can be rewritten as \[T' - B_{\ell:i} \cup A_{\ell:i} \succ S' - A_{\ell:i} \cup B_{\ell:i}.\] Plugging in $T'=T''$, $S'=S''$, $\ell=\ell'$ and $i=r'$ yields condition $(iv)$ for the updated variables. Lastly, condition $(v)$ is satisfied directly by applying the induction hypothesis. 

\textbf{Case 2:} $T' \succ S'$. Then, $S'' = T'$, $T'' = S'$, $\ell' = i+1$, $r'=r$.  For condition $(ii)$, observe that $\ell, r \in \mathbb{N}$ with $\ell < r$ clearly implies that $\ell'= i + 1 = \lfloor \frac{\ell + r}{2}\rfloor + 1\leq  \lfloor \frac{2r - 1}{2} \rfloor + 1 \leq r =r'$. Moreover, by construction $A_{i+1:r} \subseteq T' = S''$ and $B_{i+1:r} \subseteq S' = T''$ and hence $(ii)$ is satisfied. Condition $(iii)$, i.e., $S'' = T' \succ S' = T''$, is satisfied by the case condition. For condition $(iv)$, let us rewrite the induction hypothesis for condition $(iii)$ as
 \[S - A_{\ell:r} \cup (A_{\ell:i} \cup A_{i+1:r}) \succ T - B_{\ell:r} \cup (B_{\ell:i} \cup B_{i+1:r}). \] 
 Observe that $S- A_{\ell:r} \cup A_{\ell:i} = S' - B_{i+1:r} $ and $T- B_{\ell:r} \cup B_{\ell:i} = T' - A_{i+1:r}$. Hence, the above expression can be rewritten to 
 \[S' - B_{i+1:r} \cup A_{i+1:r} \succ T' - A_{i+1:r} \cup B_{i+1:r} .\] 
 Inserting $S'=T''$, $T'=S''$, $i+1=\ell'$ and $r=r'$ yields condition $(iv)$ for the updated variables. Lastly, condition $(v)$ is satisfied directly by applying the induction hypothesis. 
\end{proof}

With the help of Lemma \ref{lem:unc-log} it is easy to see that the algorithm is well-defined, more precisely, that the constructed tuple $(S,T)$ forms a feasible duel within every iteration of the while loop. It remains to show that the algorithm works correctly and its running time is bounded by $\mathcal{O}(log(|A^{(1)}|))$. 

\begin{lemma}\label{lem:uncoverStrong}
Let $A^{(1)},A^{(2)},B^{(1)},B^{(2)}$ be sets satisfying conditions $(a)$ to $(d)$. After performing $\mathcal{O}(\log(|A^{(1)}|))$ duels, \emph{Uncover} returns $(a,b)$ with $a \in A^{(1)}$, $b \in B^{(1)}$ and $(S,S') \in \mathcal{S}^{*}_{a,b}$ with either $A^{(2)} \subseteq S$ and $B^{(2)} \subseteq S'$ or $B^{(2)} \subseteq S$ and $A^{(2)} \subseteq S'$.
\end{lemma}

\begin{proof}
By Lemma \ref{lem:unc-log}, the termination of the algorithm implies that $\ell = r$. 
By statement $(ii)$ from Lemma \ref{lem:unc-log} we get that $a_\ell \in S$ and $b_\ell \in T$ holds. Moreover, conditions $(iii)$ and  $(iv)$ can be rewritten as 
\begin{equation}\label{eq:unc-log1}
(S\setminus \{a_{\ell}\}) \cup \{a_{\ell}\} \succ (T\setminus \{b_{\ell}\}) \cup \{b_{\ell}\}
\end{equation}
and
\begin{equation}\label{eq:unc-log2}
(T\setminus \{b_{\ell}\}) \cup \{a_{\ell}\} \succ (T\setminus \{a_{\ell}\}) \cup \{b_{\ell}\},
\end{equation}
respectively. Clearly, this implies that $(S \setminus \{a_{\ell}\},T \setminus \{b_{\ell}\}) \in \mathcal{S}^*_{a_{\ell},b_{\ell}}$ and hence $a_{\ell}\succ b_{\ell}$.

It is easy to see that the number of iterations of the while loop is upper bounded by the height of a balanced binary tree on $|A^{(1)}|$ elements, i.e., $\mathcal{O}(log(|A^{(1)}|))$. Since every iteration induces exactly one query, this also bounds the total number of queries. Moreover, by condition $(v)$ we have that one of $A^{(2)}$ is included in $S$ or $T$ and $B^{(2)}$ in the other one. This concludes the proof. 
\end{proof}

Clearly, Lemma \ref{lem:uncoverStrong} directly implies Lemma \ref{lem:uncover}. For this, simply call \emph{Uncover} with $A^{(2)}=B^{(2)}=\emptyset$. 
\uncoverlemma*

\medskip
\paragraph{Reducing the Number of Players to ${\mathcal{O}(k)}$}

Before formalizing the pre-processing procedure \emph{ReducePlayers} in Algorithm \ref{alg:graph_algo}, recall that algorithm maintains a dominance graph $D = (V,E)$ on the set of players. More precisely, the nodes of $D$ are the players, i.e., $V=[n]$, and there exists an arc from node $a$ to node $b$ if the algorithm has proven that $a \succ b$. The set $V_{<2k}$ is the subset of the players having an indegree smaller than $2k$ in $D$.

Additionally, we define a second graph $G_{<\m}$ as follows: The set of nodes of $G_{<\m}$ equals $V_{<\m}$ and there exists an (undirected) edge between two nodes $a, b \in V_{<\m}$ if and only if neither of the arcs $(a,b)$ or $(b,a)$ is present within the graph $D$. The algorithm now searches for a matching of size $k$ within the graph $G_{<\m}$ by calling the subroutine \textit{GreedyMatching}, formalized in Algorithm \ref{alg:GreedyMatch-a}. Let $\{(a_1,b_1), \dots, (a_k,b_k)\}$ be such a matching. In particular, this implies that the algorithm has not identified any of the relations between $a_i$ and $b_i$ yet. Hence, when calling \textit{uncover} for the (ordered) sets $A=\{a_1,\dots,a_k\}$ and $B=\{b_1,\dots, b_k\}$ (after possibly swapping $A$ and $B$), the algorithm learns about one additional pairwise relation, say $a_i \succ b_i$ and add the arc $(a_i,b_i)$ to the graph $D$. Then, the algorithm also updates $D$ to its transitive closure. The algorithm ends when it cannot find a matching of size $k$ in $G_{<\m}$ anymore. 
We formalize the idea within Algorithm \ref{alg:graph_algo}.

\begin{minipage}{0.5\textwidth}
\begin{algorithm}[H] 
\begin{algorithmic}
\STATE \textbf{Input: } a set of players $[n]$\\
\STATE \textbf{Output: } a set $S$ with $|S|\leq 6k-2$ s.t. $A^*_{\m} \subseteq S$ \\
\WHILE{$|\textit{GreedyMatching}(G_{<2k})|=k$} 
\STATE Let $\{\{a_1,b_1\},\dots, \{a_k,b_k\}\}$ be Greedy Matching \\
\STATE Set $A=\{a_1,\dots,a_k\}, B=\{b_1,\dots,b_k\}$ \\
\STATE $(a,b) \leftarrow \textit{uncover}(A,B)$ \\ 
\STATE Add $(a,b)$ to $D$, $D \leftarrow \textit{transitiveClosure}(D)$\\
\STATE Update $V_{<\m}$ and $G_{<\m}$
\ENDWHILE
\STATE \textbf{return} $V_{<\m}$
\end{algorithmic}\caption{ReducePlayers}\label{alg:graph_algo}
\end{algorithm}
\end{minipage}\hspace{0.2cm}
\begin{minipage}{0.46\textwidth}\vspace{-.8cm}
\begin{algorithm}[H] 
\begin{algorithmic}
\STATE \textbf{Input: } an undirected Graph $G=(V,E)$\\
\STATE \textbf{Output: } a matching of size at most $k$ \\
$M \leftarrow \emptyset$\\
\WHILE{$|M| < k$ and $E \neq \emptyset$} 
\STATE Pick arbitrary edge $(u,v)$ from $E$
\STATE Delete all edges incident to $u$ and $v$ from $E$
\ENDWHILE
\STATE \textbf{return }$M$
\caption{Subroutine GreedyMatching}\label{alg:GreedyMatch-a}
\end{algorithmic}
\end{algorithm}
\end{minipage}

\lemgraphalgocorrectness*

\begin{proof}
Let $\Sub$ be the set returned by \emph{ReducePlayers}. We start by proving that $A^{*}_{2k} \subseteq \Sub$. Every player not included in $\Sub$ has at least $\m$ ingoing arcs in $D$. In other words, there exist $\m$ players which dominate it. Hence, such a player is not included in $A_{\m}^*$.  

We turn to prove that $|\Sub| \leq 6k-2$: 
Any \emph{independent set} within the graph $G_{<\m}$ contains less than $\m+1$ nodes. An independent set within $G_{<\m}$ is a subset of the nodes $T \subseteq V_{<\m}$ such that no two nodes of $T$ are connected by an edge. Now, assume for contradiction that there exists an independent set $T \subseteq V_{<\m}$ within the graph $G_{<\m}$ with $|T| = \m + 1$. Consider the subgraph of $D$ induced by the set $T$, i.e., $D[T] = (T,\{(a,b) \in E \mid a,b \in T\})$. Since $T$ is an independent set within $G_{<\m}$, we know that $D[T]$ is a tournament graph, i.e., a directed graph in which any two nodes are connected by exactly one directed arc. Moreover, since $D[T]$ is transitive (since $\succ$ and hence $D$ is transitive), there exists exactly one node within $T$ with an indegree of $\m$ within the graph $D$. This is a contradiction to $T \subseteq V_{<\m}$. 

This observation is now helpful to conclude the first part of the proof. Assume for contradiction that $|V_{<\m}| \geq 6k-1$. Then the following greedy procedure lets us construct a matching of size $2k$ within the graph $G_{<\m}$. This yields a contradiction to the termination of the while loop, since every maximal matching, and in particular, a matching of size smaller than $k$ returned by \emph{GreedyMatching}, is a $1/2$-approximation of a matching with maximum cardinality. Hence, the existence of a matching with $2k$ edges yields a contradiction to the fact that \emph{GreedyMatching} did not find a matching of size $k$. We start by defining $T=V_{<\m}$ and $M=\emptyset$. Since $|T| > \m$, $T$ is not an independent set and there exists an edge between some two nodes in $T$. 
Now, pick any such edge, say $\{a,b\}$, and add it to $M$ and remove $a$ and $b$ from $T$. After $i$ rounds of this procedure, $|M| = i$ and $|T| = \m + 2(2k - i)-1$. We can repeat this procedure for $2k$ rounds and have found a matching of size $2k$, a contradiction. 

We now turn to prove the number of duels performed by the algorithm. 
In every step of the while loop, the algorithm adds one arc which was not existent before to the graph $D$. Moreover, since any selected matching never includes an edge with one of its endpoints having an indegree larger than $2k-1$, no node has an indegree higher than $\m$ after the termination of the algorithm. We can then upper bound the number of arcs within $D$ by $\m n$. 

This is also a bound for the number of iterations of the while loop. Within each iteration of the while loop the algorithm needs to make one query in order to identify the winning team and in addition it calls the subroutine \emph{uncover}. As argued within the proof of Lemma \ref{alg:graph_algo_correctness}, the \emph{uncover} subroutine induces additional $\mathcal{O}(log(k))$ queries per while loop. Summarizing, this implies that the algorithm requires $\mathcal{O}(nk \log(k))$ queries in total. 

As for the running time, we have already argued 
that the while loop does at most $\mathcal{O}(nk)$ iterations. Within the while loop the algorithm needs to run \emph{GreedyMatching} for finding a matching of size $k$ within $G_{<\m}$ and run the \emph{uncover} subroutine. While the latter step requires a running time of $\mathcal{O}(\log(k))$ as argued within Lemma \ref{alg:graph_algo_correctness}, \emph{GreedyMatching} for selecting a matching of size $k$ can be implemented in $\mathcal{O}(nk)$. 
In total, we get a running time of $\mathcal{O}(n^2k^2)$.
\end{proof}

\paragraph{Subroutines NewCut and Compare}

In Algorithm \ref{alg:subroutine} we formalize the subroutine \emph{NewCut}, which takes as input a subset of the players $\Sub \subseteq [n]$, a pair of players $a,b\in X$ and a witness $(S,T') \in \mathcal{S}_{a,b}^* \cup \mathcal{T}_{a,b}^*$ and outputs a partition of $\Sub$ into $U$ and $L$ such that $U \triangleright L$ holds. We denote by $\pi_{xy}$ the permutation on subsets that exchange players $x$ and $y$. More precisely, \[\pi_{xy}(A) =\begin{cases}  A \setminus \{x\} \cup \{y\} & \text{ if } x \in A, y \not \in A \\ A \setminus \{y\} \cup \{x\} & \text{ if } x \not\in A, y \in A \\ A & \text{else.}\end{cases}\]

\begin{algorithm}[H]
\begin{algorithmic}
\STATE \textbf{Input: } $\Sub \subseteq [n]$, a pair $a,b \in \Sub$ and $(S,T') \in \mathcal{S}^*_{a,b} \cup \mathcal{T}^*_{a,b}$ 
\STATE \textbf{Output: } Partition of $\Sub$ into $U \triangleright L$ with $a \in U$ and $b \in L$ 
\STATE Initialize $\mathcal{W} \leftarrow \{(S,T',a)\}$, $U \leftarrow \{a\}, \Sub \leftarrow \Sub \setminus \{a,b\}$ \\ 
\WHILE{$\mathcal{W}$ non-empty}
\STATE Pick $(S,T,y) \in \mathcal{W}$ and remove it from $\mathcal{W}$
\FOR{$x \in \Sub$} 
\IF{$(\pi_{xy}(S),\pi_{xy}(T')) \in \mathcal{S}^*_{xb} \cup \mathcal{T}^*_{xb}$}
\STATE add $x$ to $U$, remove $x$ from $\Sub$ \\
\STATE add $(\pi_{xy}(S),\pi_{xy}(T'),x)$ to $\mathcal{W}$
\ELSIF{$|T'|=k$ and $x \in T'$ and $(S,T'\setminus \{x\}) \in \mathcal{S}^*_{xb}$}
\STATE add $x$ to $U$ and remove it from $\Sub$ \\ 
\STATE add $(S,T'\setminus \{x\},x)$ to $\mathcal{W}$
\ENDIF
\ENDFOR
\ENDWHILE
\STATE \textbf{return} $(U,\Sub \cup \{b\})$
\end{algorithmic}
\caption{NewCut}\label{alg:subroutine}
\end{algorithm}

Before we prove the correctness of the algorithm, we introduce the following two lemmas. Strictly speaking, these are special cases of statements shown within the proof of Lemma \ref{lemma:SSTforF} for the deterministic setting. For the sake of illustration, we state and prove them here for the deterministic case again, independently of Lemma \ref{lemma:SSTforF}.

\begin{lemma} \label{lem:witness-transitivity-type1}
If $a \succ b \succ c$ and $(S,S') \in \mathcal{S}_{b,c}^*$, then $(\pi_{ab}(S),\pi_{ab}(S')) \in \mathcal{S}_{a,c}^*$. 
\end{lemma}
\begin{proof}
We distinguish two cases. First assume $a \not\in S \cup S'$. Then, \[S \cup \{a\} \succ S \cup \{b\} \succ S' \cup \{c\},\] where the first statement follows from single-player consistency and the second statement from $(S,S') \in \mathcal{S}^*_{bc}$. Moreover, \[S' \cup \{a\} \succ S' \cup \{b\} \succ S \cup \{c\},\] where again the first statement follows from single-player consistency and the second one from $(S,S') \in \mathcal{S}_{bc}^*$.

If $a \in S \cup S'$, assume wlog that $a \in S$. Then, $\pi_{ab}(S) = S \setminus \{a\} \cup \{b\}$ and $\pi_{ab}(S') = S'$. We get \[\pi_{ab}(S) \cup \{a\} = S \cup \{b\} \succ S' \cup \{c\}\] and \[S' \cup \{a\} \succ S' \cup \{b\} \succ S \cup \{c\} = S \setminus \{a\} \cup \{a\} \cup \{c\} \succ S \setminus \{a\} \cup \{b\} \cup \{c\} = \pi_{ab}(S) \cup \{c\},\] where the first and last statement follow from single player consistency and the second statement from $(S,S') \in \mathcal{S}_{bc}^*$. Summarizing, $(\pi_{ab}(S),\pi_{ab}(S')) \in \mathcal{S}_{ac}^*$.
\end{proof}
\begin{lemma} \label{lem:witness-transitivity-type2}
If $a \succ b \succ c $ and $(S,T) \in \mathcal{T}_{b,c}^*$, then $(\pi_{ab}(S),\pi_{ab}(T)) \in \mathcal{T}_{a,c}^*$ or $(S,T\setminus \{a\}) \in \mathcal{S}_{a,c}^*$. 
\end{lemma}
\begin{proof}
We distinguish three cases. First, assume that $a \not \in S \cup T$. Then, \[S \cup \{a\} \succ S \cup \{b\} \succ T \succ S \cup \{c\},\] where the first statement follows from single player consistency and the second and third from $(S,T) \in \mathcal{T}_{bc}^*$. Next, assume $a \in S$. Then, $\pi_{ab}(S) = S \setminus \{a\} \cup \{b\}$ and we get \[\pi_{ab}(S) \cup \{a\} = S \cup \{b\} \succ T \succ S \cup \{c\} = S \setminus \{a\} \cup \{a\} \cup \{c\} \succ S \setminus \{a\} \cup \{b\} \cup \{c\} = \pi_{ab}(S) \cup \{c\}.\] Hence, $(\pi_{ab}(S),\pi_{ab}(T)) \in \mathcal{T}_{ab}^*$. Finally, assume $a \in T$.  
We get, \[S \cup \{a\} \succ S \cup \{b\} \succ T \setminus \{a\} \cup \{a\} \succ T \setminus \{a\} \cup \{c\}, \] where the first and last statement follow from single player consistency and the second statement from $(S,T) \in \mathcal{T}_{bc}^*$. 
Moreover, \[T \setminus \{a \} \cup \{a\} \succ S \cup \{c\},\] which follows from $(S,T) \in \mathcal{T}_{bc}^*$. Summarizing, $(S,T \setminus \{a\}) \in \mathcal{S}_{ac}^*$.
\end{proof}

Having these two lemmas, we are ready to prove the correctness of the \emph{NewCut} subroutine. 

\newcutcorrect*

\begin{proof}
Let $\Sub$ be the original set of players given as input to the algorithm, and $U$ and $L$ the returned sets. We denote by $\Sub'$ and $U'$ the corresponding sets maintained and modified by the algorithm during its execution. To see that $U$ and $L$ form a partition of $V$, observe that $U'$ and $\Sub'$ form a partition of $\Sub \setminus \{b\}$ during the entire execution of the algorithm. 

We turn to show that $U \triangleright L$. Assume for contradiction that there exists $c \in L$ and $d \in U$ with $c \succ d$. Since $d \in U$ we know that the algorithm found a witness for $d \succ b$ which we denote by $(S,T')$ and added $(S,T',d)$ to the list $\mathcal{W}$. Moreover, as $c \in L$, the algorithm selected $x=c$ in the for loop when $(S,T',d)$ was picked from $\mathcal{W}$. Now, if $|T'| = k-1$, we know that $(S,T') \in \mathcal{S}_{d,b}^*$ and can apply Lemma \ref{lem:witness-transitivity-type1} which yields $(\pi_{cd}(S),\pi_{cd}(T')) \in S^*_{c,b}$. This is a contradiction, as otherwise $c$ would have been added to $U'$ at this point. If $|T'|=k$, we can apply Lemma \ref{lem:witness-transitivity-type2}, yielding that either $(\pi_{cd}(S),\pi_{cd}(T')) \in \mathcal{T}^*_{c,b}$ or $(S,T' \setminus \{c\}) \in \mathcal{S}^*_{c,b}$, both of which cannot be as $c \not \in U'$ at the end of the algorithm. This completes the proof of correctness. 

It remains to bound the number of duels performed. Since the number of duels performed in every iteration of the for loop is constant, it suffices to bound the number of iterations of the for loop. As the algorithm adds at most $|\Sub| - 1$ elements to $\mathcal{W}$ and for each element the for loop runs at most $|\Sub|-2$ times, the number of duels can be bounded by $\mathcal{O}(|\Sub|^2)$. 
\end{proof}

We now turn to formalize the subroutine \emph{Compare} within Algorithm \ref{sub:compare}. 
\begin{algorithm}[H] 
\begin{algorithmic}
\STATE \textbf{Input: } tuple $(a,b)$, witness $(S,S') \in \mathcal{S}_{ab}^*$ and $C \subseteq S$, $D \subseteq S'$ with $|C| = |D|$
\smallskip
\IF {$S \setminus C \cup D \cup \{a\} \succ S' \setminus D \cup C \cup \{b\}$ and $S' \setminus D \cup C \cup \{a\} \succ S \setminus C \cup D \cup \{b\}$}
\STATE \textbf{return} True 
\ELSE \STATE\textbf{return} False
\ENDIF
\end{algorithmic}
\caption{Compare}\label{sub:compare}
\end{algorithm}

\compare*
\begin{proof}
For the sake of brevity we define $\bar{S} = S \setminus C$ and $\bar{S}' = S \setminus D$.
Recall that from $(S,S') \in \mathcal{S}_{a,b}$ we get that $(i) \; \Bar{S} \cup C \cup \{a\} \succ \Bar{S}' \cup D \cup \{b\}$ and $(ii) \; \Bar{S}' \cup D \cup \{a\} \succ \Bar{S} \cup C \cup \{b\}$ hold. 
Recall that we are considering additive total orders. 
For any set $A \subseteq [n]$ we define $v(A) = \sum_{a \in A} v(a)$. 
Then, we can rewrite $(i)$ and $(ii)$ to \[(i)\; v(\bar{S}) + v(C) + v(a) > v(\bar{S}') + v(D) + v(b)\] and \[(ii)\; v(\bar{S}') + v(D) + v(a) > v(\bar{S}') + v(C) + v(b).\]

Then, we distinguish two cases. \\
\textbf{Case 1.} $(iii)\; \Bar{S} \cup D \cup \{a\} \succ \Bar{S}' \cup C \cup \{b\}$ and $(iv) \;\Bar{S}' \cup C \cup \{a\} \succ \Bar{S} \cup D \cup \{b\}$. Similarly to before, we can rewrite $(iii)$ and $(iv)$ to \[(iii)\; v(\bar{S}) + v(D) + v(a) > v(\bar{S}') + v(C) + v(b)\] and \[(iv)\; v(\bar{S}') + v(C) + v(a) > v(\bar{S}) + v(D) + v(b).\]

Then, from adding $(ii)$ and $(iii)$ we get that \[v(a) - v(b)> v(C) - v(D)\]
and from adding $(i)$ and $(iv)$ we get that 
\[v(a) - v(b) > v(D) - v(C).\]
Summarizing, this yields $v(a) - v(b) > |v(C)| - |v(D)|$. 

\textbf{Case 2.} $(v) \; \Bar{S}' \cup C \cup \{b\} \succ \Bar{S} \cup D \cup \{a\}$

In that case, observe that the quartet $(C,D,\Bar{S} \cup \{a\}, \Bar{S}' \cup \{b\})$ satisfies the requirements for the \emph{Uncover} subroutine due to equation $(i)$ and $(v)$. Hence, \emph{Uncover} will return a dominance of some player in $C$ towards some player in $D$ together with a witness for this relationship. 

\textbf{Case 3.} $(vi) \; \Bar{S} \cup D \cup \{b\} \succ \Bar{S}' \cup C \cup \{a\}$

In that case, observe that the quartet $(D,C,\Bar{S} \cup \{b\}, \Bar{S}' \cup \{a\})$ satisfies the requirements for the \emph{Uncover} subroutine due to equation $(ii)$ and $(vi)$. Hence, \emph{Uncover} will return a dominance of some player in $D$ towards some player in $C$ together with a witness for this relationship. 
\end{proof}

\paragraph{Algorithm CondorcetWinning}

Recall that the algorithm maintains a partition of the players into a weak ordering, i.e., $\mathscr{T} = \{T_1, \dots, T_{\ell}\}$ with $T_1 \triangleright T_2 \triangleright \dots \triangleright T_{\ell}$. We introduce the short-hand notation $T_{\leq j} = \bigcup_{m \in [j]} T_m$ and $T_{< j} = \bigcup_{m\in[j-1]} T_m$.
After the application of the preprocessing procedure \emph{ReducePlayers}, this partition consists of one set, namely $\mathscr{T} = \{T_1\}$, where $|T_1| \in \mathcal{O}(k)$ and $A^*_{2k} \subseteq T_1$. 
At any point in the execution of the algorithm, we are especially interested in two indices, namely $i_k \in [\ell]$ such that $|T_{<i_k}|<k<|T_{\leq i_k}|$ and similarly  $i_{2k} \in [\ell]$ such that $|T_{<i_{2k}}|<2k<|T_{\leq i_{2k}}|.$ In case one of these indices does not exist, this implies that we have either identified the set $A^*_k$ or $A^*_{2k}$. In the first case, we have found a Condorcet winning team and in the second case Observation \ref{obs:2k} implies that we can find one by performing one additional duel. For the sake of brevity, we disregard this case from now on. 

Assuming $i_k$ is defined, observe that all players from $T_{<i_k}$ are guaranteed to be among the top-k players. On the other hand, among the players from $T_{i_k}$ some belong to $A^*_k$ and others do not. 
The main idea of the algorithm will then be to, at any given time, take some $k$-sized prefix of $\mathscr{T}$, i.e., a subset including $T_{<i_k}$ that is included in $T_{\leq i_{k}}$ and either proving that this prefix is a Condorcet winning team, or showing that the partition $\mathscr{T}$ can be refined. 

In the following we distinguish the cases that $i_k \neq i_{2k}$ and $i_k = i_{2k}$. For the first case we give the algorithm \emph{CondorcetWinning1} and for the latter case the algorithm \emph{CondorcetWinning2}. Observe that, once the \emph{CondorcetWinning1} called \emph{CondorcetWinning2} (which implies $i_k \neq i_{2k}$) this will be true until the termination of the algorithm. 

\medskip
\paragraph{CondorcetWinning1}
The algorithm starts by partitioning the set $T_{< i_{k}}$ into two sets $U_1$ and $U_2$, where $U_1$ is a prefix of $T_{<i_k}$ of size $|T_{\leq i_k}|-2k$. It partitions the set $T_{i_k}$ into five sets $X,Y,W_1,W_2,$ and $Z$. In particular it is known that $(U_1 \cup U_2) \triangleright (X \cup Y \cup W_1 \cup W_2 \cup Z)$ but no relation among any pair in $T_{i_k}$ is known. Regarding the sizes of the sets it holds that $|U_i| = |W_i|$ for $i \in \{1,2\}$, $|X|=|Y| = k - |U_1|-|U_2|$ and $|U_1| = |Z|$. The main aim of the algorithm will be to define $0< \epsilon_1 < \epsilon_2$ and prove that the following statements are true:
\begin{enumerate}[label=(\roman*)]
    \item $|v(X) - v(Y)| < \epsilon_1$
    \item $|v(a) - v(b)| < \epsilon_2$ for all $a \in Y \cup W_1 \cup W_2$ and $b \in Z$, and 
    \item there exist $u_1, \dots, u_{|Z|+1} \in U_1 \cup U_2$ as well as $w_1, \dots, w_{|Z|+1} \in W_1 \cup W_2$ such that 
    \begin{enumerate}[label=(\alph*)]
        \item $v(u_1) - v(w_1) \geq \epsilon_1$ and
        \item $v(u_i) - v(w_i) \geq \epsilon_2$ for all $i \in \{2, \dots, |Z|+1\}$.
    \end{enumerate}
\end{enumerate}
With these three statements we can show that $U_1 \cup U_2 \cup X$ is a Condorcet winning team. More precisely, one can show that $v(U_1 \cup U_2 \cup X) - v(W_1 \cup W_2 \cup Y) > |Z| \cdot \epsilon_2$ and $v(W_1 \cup W_2 \cup Y) - v(B^*) > - |Z| \cdot \epsilon_2$, where $B^*$ is the best response towards $U_1 \cup U_2 \cup X$, i.e., $B^*$ simply contains the best $k$ players from $[n] \setminus (U_1 \cup U_2 \cup X)$. See Figure \ref{fig:1} for an illustration of the argument. 

It remains to sketch how the algorithm defines $\epsilon_1,\epsilon_2$ and proves $(i)-(iii)$. %, 
The algorithm starts by checking whether \emph{Uncover} can be applied to the sets $A^{(1)}=U_2,A^{(2)}=X \cup Z, B^{(1)}=W_2,B^{(2)}=Y\cup W_1$. If this is not the case, a relation between a pair in $A^{(2)}$ and $B^{(2)}$ can be found and the partition can be refined by applying \emph{NewCut}. Otherwise, let $\Bar{u} \in U_2$ and $\Bar{w} \in W_2$ be the returned pair from \emph{Uncover}. For the sake of brevity we assume for now that the entire indifference class of $\Bar{u}$ in $\mathscr{T}$ is included in $U_2$. Then, using \emph{Compare}, the algorithm checks whether $|v(X) - v(Y)| < v(\Bar{u}) - v(\Bar{w})$ and whether $|v(a) - v(b)| < v(\Bar{u}) - v(\Bar{w})$ for all $a \in W_1 \cup W_2 \cup Y$ and $b \in Z$. The algorithm repeats the process by replacing $\Bar{w}$ by all $w \in W_1$. If any of the calls to \emph{Compare} returned \emph{False}, then we show that the partition can be refined. Otherwise, we have shown that conditions $(i)-(iii)$ are satisfied for $\epsilon_1 = v(\Bar{u}) - v(w^*_1)$ and $\epsilon_2 = v(\Bar{u}) - v(w_2^*)$, where $w_1^*$ and $w_2^*$ are the best and second best players from $W_1 \cup \{\Bar{w}\}$, respectively. For the case when not the entire indifference class of $\bar{u}$ is included in $U_2$, we still have to exchange $\bar{u}$ by other players from its indifferent class which are included in $U_1$.

\begin{lemma}\label{lem:Condorcet1}
After performing $\mathcal{O}(k^5)$ many duels, \emph{CondorcetWinning1} has identified a Condorcet winning team or called \emph{CondorcetWinning2}.
\end{lemma}

\begin{proof}
In part I we show that the algorithm is well-defined and that, within line \ref{line:part1},\ref{line:part2},\ref{line:part3}, \ref{line:part4}, \ref{line:part5}, and \ref{line:part6}, a refined partition can indeed be found. In part II we show that, if the algorithm outputs a team, this team is indeed Condorcet winning. Lastly, in part III we argue about the bound on the number of duels performed.  

\textbf{Part I.} We show the first two statements by going through the algorithm line by line. 

We start by showing that in line \ref{line:firstIF}, the two queries are feasible. First observe that by construction, the sets $U_1,U_2,X,Y,W_1,W_2,$ and $Z$ are disjoint. Moreover, $|U| = |W|$, $|U_1| = |W_1|$, and hence $|U_2| = |W_2|$. Also, $|X|=|Y|$ and $|W_1| = |Z|$. In total, we get that $|W_2| + |Y| + |W_1| = |U_1| + |X| + |Z| = |U|+|X| = k$ and the same holds for the other query as well. 

Next, we show that in line \ref{line:part1}, the partition $\mathscr{T}$ can indeed be refined. Consider wlog the case when $W_2 \cup (Y \cup W_1) \succ U_2 \cup (X \cup Z)$. Then, since $U_2 \triangleright W_2$ we know that $U_2 \cup (Y \cup W_1) \succ W_2 \cup (X \cup Z)$ needs to hold. Hence, $\mathrm{Uncover}(Y\cup W_1,X \cup Z,W_2,U_2)$ returns a pair $(a,b)$ with $a \in Y \cup W_1$ and $b \in X \cup Z$ together with a witness $(S,S') \in \mathcal{S}_{a,b}$. Since $a,b \in T_{i_k}$, we can call $\mathrm{NewCut}(\mathscr{T},(a,b),(S,S'))$ which returns a refined partition. An analogous argument holds for the case $W_2 \cup (X \cup Z) \succ U_2 \cup (Y \cup W_2)$. 

We turn to show that the input for the $\mathrm{Uncover}$ subroutine is valid in line \ref{line:warmstart}. Since the condition in line \ref{line:firstIF} is not satisfied, we know that $U_2 \cup (X \cup Z) \succ W_2 \cup (Y \cup W_1)$ and $U_2 \cup (Y \cup W_1) \succ W_2 \cup (X \cup Z)$. This suffices to show that $(U_2, W_2,(X \cup Z),(Y \cup W_1))$ is a valid input for $\mathrm{Uncover}$. Hence, for the returned pair $(\bar{u},\bar{w})$ is holds that $\bar{u} \in U_2$ and $\bar{w} \in W_2$. Moreover, we can assume in the following wlog that $(X \cup Z) \subseteq S$ and $(Y \cup W_1) \subseteq S'$. 

We continue with the situation in line \ref{line:part2} and show that a refined partition can be found. We distinguish two cases. 

\textbf{Case 1} $(S,S'') \in \mathcal{S}_{\bar{u},w}$. This implies $(i) \; S \cup \{\bar{u}\} \succ S'' \cup \{w\}$ and $(ii) \; S'' \cup \{\bar{u}\} \succ S \cup \{w\}$. Moreover, from $(S,S'') \not \in \mathcal{S}_{u,w}$ we know that either $ (iii) \; S \cup \{w\} \succ S'' \cup \{u\} $ or $ (iv) \;S'' \cup \{w\} \succ S \cup \{u\}$ is true. Assume without loss of generality that $(iii)$ holds. Then, together with $(ii)$ we get that $S'' \cup \{\bar{u}\} \succ S \cup \{w\} \succ S'' \cup \{u\}$, hence $\bar{u} \succ u$ and in particular $(S \cup \{w\},S'') \in \mathscr{T}_{\bar{u},u}$. Since $\bar{u}$ and $u$ are from the same indifference class of $\mathscr{T}$, calling $\mathrm{NewCut2}(\mathscr{T},(\bar{u},u),(S \cup \{w\},S''))$ returns a refined partition. An analogous argument holds when $(iv)$ is true.
\begin{algorithm}[H] 
\begin{algorithmic}[1]
\STATE \textbf{Input: } a partition of $[n]$ into $T_1 \triangleright T_2 \triangleright \dots \triangleright T_{\ell}$
\STATE \textbf{Output: } a CondorcetWinning Team \\
\smallskip
\IF{$i_k \neq i_{2k}$}
\STATE \textbf{return} $\mathrm{CondorcetWinning2}(\mathscr{T})$
\ENDIF
\STATE Set $U \leftarrow T_{<i_k}$
\STATE Set $X$ and $Y$ to be two disjoint, $(k-|U|)$-sized subsets of $T_{i_k}$
\STATE Set $W$ to be a $|U|$-sized subset of $T_{i_k} \setminus X \setminus Y$
\STATE Set $Z$ to be $T_{i_k} \setminus X \setminus Y \setminus W$
\STATE Set $W_1$ to be a $|Z|$-sized subset of $W$ and $W_2 \leftarrow W \setminus W_1$
\STATE Set $U_1$ to be a $|Z|$-sized prefix of $U$ and $U_2 \leftarrow U \setminus U_1$
\IF{$W_2 \cup (Y \cup W_1) \succ U_2 \cup (X \cup Z)$ or $ W_2 \cup (X \cup Z)\succ U_2 \cup (Y \cup 
W_1)$} \label{line:firstIF}
\STATE {\textbf{return} $\mathrm{CondorcetWinning(refined Partition)}$} \label{line:part1}
\ENDIF
\STATE $(\bar{u},\bar{w}),(S,S') \leftarrow \mathrm{Uncover(U_2,W_2,(X \cup Z),(Y \cup W_1))}$ \label{line:warmstart}
\STATE Let $\bar{T}$ be indifference class of $\bar{u}$ in $\mathscr{T}$
\FOR{$u \in \bar{T} \cap U_1 \cup \{\bar{u}\}$}
\FOR{$w \in W_1 \cup \{\bar{w}\}$}
\STATE $S'' \leftarrow f_{\bar{w},w}(S')$ 
\IF{$(S,S'') \not \in \mathcal{S}_{u,w}$}
\STATE {\textbf{return} $\mathrm{CondorcetWinning(refined Partition)}$} \label{line:part2}
\ENDIF
\IF{$\mathrm{Compare}((u,w),(S,S''),(X,Y))$ not true}
\STATE {\textbf{return} $\mathrm{CondorcetWinning(refined Partition)}$} \label{line:part3}
\ENDIF
\FOR{$z \in Z$}
\FOR{$q \in S'' \cap (W \cup Y)$}
\IF{$\mathrm{Compare}((u,w),(S,S''),(\{z\},\{q\}))$ not true} \label{line:comparezq}
\STATE {\textbf{return} $\mathrm{CondorcetWinning(refined Partition)}$} \label{line:part4}
\ENDIF
\ENDFOR
\ENDFOR
\STATE $(Q,Q') \leftarrow (S \setminus Z \cup \pi_{w^*,w}(W_1), S'' \setminus \pi_{w^*,w}(W_1) \cup Z)$
\IF{$(Q,Q') \not \in S_{u,w}$}
\STATE {\textbf{return} $\mathrm{CondorcetWinning(refined Partition)}$} \label{line:part5}
\ENDIF
\FOR{$z \in Z$}
\FOR{$w' \in Q \cap W_2$}
\IF{$\mathrm{Compare}((u,w),(Q,Q'),(\{w'\},\{z\}))$ not true} \label{line:comparezq2}
\STATE {\textbf{return} $\mathrm{CondorcetWinning(refined Partition)}$} \label{line:part6}
\ENDIF
\ENDFOR
\ENDFOR
\ENDFOR
\ENDFOR
\STATE \textbf{return} $U \cup X$
\end{algorithmic}
\caption{CordorcetWinning1}\label{alg:condorcet2}
\end{algorithm}

\textbf{Case 2} $(S,S'') \not\in \mathcal{S}_{\bar{u},w}$. Then, either $(i)\; S \cup \{w\} \succ S'' \cup \{\bar{u}\}$ or $(ii) \; S'' \cup \{w\} \succ S \cup \{\bar{u}\}$ holds while both is not possible as $\bar{u} \succ w$. %Moreover, from $(S,S'') \not \in \mathcal{S}_{u,w}$ we know that either $ (iii) \; S \cup \{w\} \succ S'' \cup \{u\} $ or $ (iv) \;S'' \cup \{w\} \succ S \cup \{u\}$ is true. First assume that $(i)$ holds. 
First, assume $(i)$ is true. Then, from $(S,S') \in \mathcal{S}_{\bar{u},\bar{w}}$, we know that $(iii) \; S' \cup \{\bar{u}\} \succ S \cup \{\bar{w}\}$. Reformulating $(i)$ to $S \cup \{w\} \succ S' \setminus \{w\} \cup \{\bar{u}\} \cup \{\bar{w}\}$ and $(iii)$ to $S' \setminus \{w\} \cup \{\bar{u}\} \cup \{w\} \succ S \cup \{\bar{w}\}$ shows that $w \succ \bar{w}$ and in particular $(S,S' \setminus \{w\} \cup \{\bar{u}\}) \in \mathcal{S}_{w,\bar{w}}$. As $w$ and $\bar{w}$ are contained in the same indifference class of $\mathscr{T}$, calling $\mathrm{NewCut}(\mathscr{T},(w,\bar{w}),(S, S' \setminus \{w\} \cup \{\bar{u}\}))$ refines the partition. Second, assume that $(ii)$ holds. However, from $(S,S') \in \mathcal{S}_{\bar{u},\bar{w}}$ we know that $(iv) \; S \cup \{\bar{u}\} \succ S' \cup \{\bar{w}\}$ is true. As $S'' \cup \{w\} = S' \cup \{\bar{w}\}$ this yields a contradiction to $(ii)$. 

We prove that we can find a refined partition within line \ref{line:part3}. When $\mathrm{Compare}((u,w),(S,S''),(X,Y))$ is not true, then one call to $\mathrm{Uncover}(X,Y,S\setminus X, S'' \setminus Y)$ returns a pair $(x,y)$ with $x \succ y$ (or vice versa) and a witness $(P,P') \in \mathcal{S}_{x,y}$ (or $(P,P') \in \mathcal{S}_{y,x}$) (as shown within Lemma \ref{lem:compare}). Since $x$ and $y$ are from the same indifference class of $\mathscr{T}$, namely $T_{i_k}$, the algorithm can call $\mathrm{NewCut}(\mathscr{T},(x,y),(P,P'))$ and obtain a refined partition. 

We continue with the situation in line \ref{line:part4}. When $\mathrm{Compare}((u,w),(S,S''),(\{z\},\{w'\}))$ is not true, then a call to $\mathrm{Uncover}(\{z\},\{w'\},S \setminus \{z\},S'' \setminus \{w'\})$ returns the pair $(z,w')$  (or $(w',z)$) and a witness $(P,P') \in \mathcal{S}_{z,w'}$ (or $(P,P') \in \mathcal{S}_{w',z}$). Since $z$ and $w'$ are from the same indifference class of $\mathscr{T}$, namely $T_{i_k}$, the algorithm can call $\mathrm{NewCut}(\mathscr{T},(z,w'),(P,P'))$ and obtain a refined partition.

We turn to prove that we can find a refined partition within line \ref{line:part5}. From $(Q,Q') \not\in \mathcal{S}_{u,w}$ we know that either $(i) \; Q \cup \{w\} \succ Q' \cup \{u\}$ or $(ii) \; Q' \cup \{w\} \succ Q \cup \{u\}$ while both are not possible as $u \succ w$. First, assume that $(i)$ holds. From $(S,S'') \in \mathcal{S}_{u,w}$ we get in particular that $(iii) \; S'' \cup \{u\} \succ S \cup \{w\}$ holds. Rewriting $(i)$ as $\pi_{\bar{w},w}(W_1) \cup S \setminus Z \cup \{w\} \succ Z \cup S'' \setminus \pi_{\bar{w},w}(W_1) \cup \{u\}$ and $(iii)$ as $\pi_{\bar{w},w}(W_1) \cup S'' \setminus \pi_{\bar{w},w}(W_1) \cup \{u\}\succ Z \cup S \setminus Z   \cup \{w\}$ establishes that we can call $\mathrm{Uncover}(\pi_{\bar{w},w}(W_1),Z,S'' \setminus \pi_{\bar{w},w}(W_1) \cup \{u\},S \setminus Z \cup \{w\} )$ which returns a pair $(\hat{w},\hat{z})$ with $\hat{w} \in \pi_{\bar{w},w}(W_1)$ and $\hat{z} \in Z$ together with a witness for their relation. As $\hat{w}$ and $\hat{z}$ are from the same indifference class of $\mathscr{T}$ we can call $\mathrm{NewCut}$ to refine the partition. The case when $(ii)$ follows by an analogous argument. 

Lastly, we show that we can find a refined partition within line \ref{line:part6}. $\mathrm{Compare}((u,w),(Q,Q'),(\{w'\},\{z\}))$ is a valid query as, for starters, $w' \in Q$ and $z \in Q'$. Moreover, $(Q,Q') \in \mathcal{S}_{u,w}$. Hence, if $\mathrm{Compare}$ returns False, then $\mathrm{Uncover}(\{w'\},\{z\},Q \setminus \{w'\},Q' \setminus \{z\})$ returns the pair $(w',z)$ (or $(z,w')$) together with a witness from $\mathcal{S}_{w',z}$ (or $\mathcal{S}_{z,w'}$). As $z$ and $w'$ are from the same equivalence class of $\mathscr{T}$, we can call the $\mathrm{NewCut}$ and obtain a refined partition.   

\textbf{Part II.} We now show that the set returned by $\mathrm{CondorcetWinning(\mathscr{T})}$ is indeed a Condorcet winning team. If, at some point of the algorithm $i_k \neq i_{2k}$, then the statement follows from Lemma \ref{lem:Condorcet2}. Otherwise, the algorithm returns $U \cup X$ which implies that within the last call of $\mathrm{CondorcetWinning}$ none of the if conditions was satisfied. We show in the following that this implies that $U \cup X$ is a Condorcet winning team. 

We define 
\begin{align*}
w^*_1 &= \argmax_{w \in W_1 \cup \{\bar{w}\}} v(w), \\ 
w_2^* &= \argmax_{w \in W_1 \cup \{\bar{w}\} \setminus \{w_1^*\}} v(w), \text{ and } \\ u^* &= \argmin_{u \in \bar{T} \cap U_1} v(u).
\end{align*}
Moreover, $\epsilon_1 = v(u^*) - v(w^*_1) \text{ and } \epsilon_2 = v(u^*) - v(w^*_2)$. 

We claim that 
\begin{enumerate}[label=(\roman*)]
    \item $|v(X) - v(Y)| < \epsilon_1$, and \label{it:diffXY}
    \item $|v(a) - v(b)| < \epsilon_2$ for all $a \in Y \cup W$ and $b \in Z$. \label{it:diffaz}
\end{enumerate}

For \ref{it:diffXY} observe that there was a point within the iteration of the algorithm when $u = u^*$ and $w=w^*_1$. Moreover, the algorithm called $\mathrm{Compare}((u,w),(S,S''),(X,Y))$ which returned true. As we have argued for the subroutine $\mathrm{Compare}$, this implies $\epsilon_1 = v(u^*) - v(w^*_1) > |v(X) - v(Y)|$.

To show \ref{it:diffaz}, we distinguish three cases. Let $a \in Y \cup W$ and $z \in Z$. 

\textbf{Case 1.} $a = w_1^*$. Then, there was a point within the iteration of the algorithm when $u=u^*, w=w_2^*, q = w_1^*=a$ and $z = b$. As $\mathrm{Compare}((u,w),(S,S'),(\{z\},\{q\}))$ returned true in line \ref{line:comparezq}, we know that \[|v(a)- v(b)| < v(u^*)-v(w^*_2) = \epsilon_2.\] 

\textbf{Case 2.} $a \neq w_1^*, a \in S$. Then, there was a point within the iteration of the algorithm when $u=u^*, w=w_1^*, q = a$ and $z = b$. As $\mathrm{Compare}((u,w),(S,S'),(\{z\},\{q\}))$ returned true in line \ref{line:comparezq}, we know that \[|v(a)- v(b)| < v(u^*)-v(w^*_1) = \epsilon_1 < \epsilon_2.\] 

\textbf{Case 3.} $a \neq w_1^*, a \in S'$. Then, there was a point within the iteration of the algorithm when $u=u^*, w=w_1^*, q = a$ and $z = b$. As $\mathrm{Compare}((u,w),(Q,Q'),(\{z\},\{q\}))$ returned true in line \ref{line:comparezq2}, we know that \[|v(a)- v(b)| < v(u^*)-v(w^*_1) = \epsilon_1 < \epsilon_2.\] 
 
Lastly, we show that \ref{it:diffXY} and \ref{it:diffaz} suffice to prove that $U \cup X$ is a Condorcet winning team. To this end let $B^*$ be the best response against $U \cup X$. Observe that $B^* \subseteq Y \cup W \cup Z$. 

We start by showing 
\begin{align*}
    \quad &v(U \cup X) - v(W \cup Y)\\ 
    &= v(U_1 \cup \{\bar{u}\} \setminus \{u^*\}) + v(u^*) + v(U_2 \setminus \{\bar{u}\}) +v(X) \\ & \quad  - v(w^*_1) - v(W_1 \cup \{\bar{w}\} \setminus \{w^*_1\}) - v(W_2 \setminus \{\bar{w}\}) - v(Y) \\ 
    & = v(X) - v(Y) + v(u^*) - v(w^*_1) +  v(U_1 \cup \{\bar{u}\} \setminus \{u^*\})\\ & \quad  - v(W_1 \cup \{\bar{w}\} \setminus \{w^*_1\}) + v(U_2 \setminus \{\bar{u}\}) - v(W_2 \setminus \{\bar{w}\}) \\ 
    & > - \epsilon_1 + v(u^*) - v(w^*_1) +  v(U_1 \cup \{\bar{u}\} \setminus \{u^*\}) \\ & \quad - v(W_1 \cup \{\bar{w}\} \setminus \{w^*_1\}) + v(U_2 \setminus \{\bar{u}\}) - v(W_2 \setminus \{\bar{w}\}) \\ 
    & > - \epsilon_1 + \epsilon_1 +  v(U_1 \cup \{\bar{u}\} \setminus \{u^*\}) - v(W_1 \cup \{\bar{w}\} \setminus \{w^*_1\}) + v(U_2 \setminus \{\bar{u}\}) - v(W_2 \setminus \{\bar{w}\}) \\
    & > - \epsilon_1 + \epsilon_1 + |Z| \cdot \epsilon_2 + v(U_2 \setminus \{\bar{u}\}) - v(W_2 \setminus \{\bar{w}\}) \\
    & > - \epsilon_1 + \epsilon_1 + |Z| \cdot \epsilon_2 + 0 \\ 
    & = |Z| \cdot \epsilon_2.
\end{align*}
The first inequality follows by \ref{it:diffXY}, the second by the definition of $\epsilon_1$, the third by the definition of $\epsilon_2$ and the fact that $|u(U_1 \cup \{\bar{u}\} \setminus \{u^*\})| = |v(W_2 \cup \{\bar{w}\} \setminus \{w_1^*\})| = |Z|$, and the last by the fact that $U_2 \triangleright W_2$. 

In addition, we get 

\begin{align*}
    v(W \cup Y) - v(B^*) & = v(W \cup Y \setminus B^*) - v(B^* \cap Z) \\ 
    & > - |Z| \cdot \epsilon_2, 
\end{align*}
where the inequality follows from the fact that $|v(W \cup Y)| = |v(B^* \cap Z)| < |Z|$ and \ref{it:diffaz}. 

Summing up the two inequalities yields \[v(U \cup X) - v(B^*) > 0,\]
which concludes this part of the proof. \\

\textbf{Part III.} It remains to argue about the number of duels performed by \emph{CondorcetWinning1} until it calls \emph{CondorcetWinning2} or returns a team. We first observe that the partition $\mathscr{T}$ can be refined at most $\mathcal{O}(k)$ times. Also, the number of calls to \emph{Uncover} can be bounded by $\mathcal{O}(k)$, since, \emph{Uncover} is either called just before a refinement (hidden within any of the lines saying ``refinedPartition'') or within line \ref{line:warmstart}. In the following, we will therefore bound the number of duels done within one recursive call of \emph{CondorcetWinning1}. To this end, observe that checking whether some tuple is a subsets witness as well as calling \emph{Compare} requires $\mathcal{O}(1)$ duels. Clearly, the number of times these operations are performed within one recursive call (before the next call is initiated) can be bounded by $\mathcal{O}(k^4)$. Putting all of this together yields that the number of duels can be bounded by $\mathcal{O}(k^5)$.  
\end{proof}

\paragraph{CondorcetWinning2}
We continue by formalizing the second case of the algorithm, which is formalized within Algorithm \ref{alg:condorcet2}. Since the approach is significantly easier than the one of \emph{CondorcetWinning1}, we directly give the proof.  

\begin{algorithm}[tb] 
\begin{algorithmic}[1]
\STATE \textbf{Input: } a partition of $[n]$ into $T_1 \triangleright T_2 \triangleright \dots \triangleright T_{\ell}$ with $i_{k} \neq i_{2k}$ 
\STATE \textbf{Output: } a CondorcetWinning Team \\
\smallskip
\STATE $j \leftarrow \min \{k - |T_{<i_k}|, |T_{\leq i_k}| - k\}$
\STATE Set $X$ and $Y$ to be two disjoint, $j$-sized subsets of $T_{i_k}$
\STATE Set $W \leftarrow T_{i_k} \setminus X \setminus Y$
\STATE Set $L \leftarrow \emptyset$
\STATE $(*)$ Set $Z$ to be a subset of $T_{i_{2k}}\setminus L$ of size $2k - |T_{<i_{2k}}|$
\WHILE{$|L|< |T_{\leq i_{2k}}| - 2k + 1$}
\IF{$|T_{\leq i_k}| - k<k-|T_{<i_k}|$}
\STATE $U \leftarrow T_{<i_k} \cup W, \;V \leftarrow (T_{>i_k} \cap T_{<i_{2k}}) \cup Z$ \\ 
\ELSE \STATE $U \leftarrow T_{<i_k}, \; V \leftarrow W \cup (T_{>i_k} \cap T_{<i_{2k}}) \cup Z$ \\
\ENDIF
\IF{$V \cup Y \succ U \cup X$ or $V \cup X \succ U \cup Y$} \label{line:firstquery}
\STATE \textbf{return} $\mathrm{CondorcetWinning}(\mathrm{refinedPartition})$\label{line:partition1}
\ENDIF
\STATE $(u,v),(S,S') \leftarrow \mathrm{Uncover}(U,V,X,Y)$ \label{line:warmstart1}
\IF{$\mathrm{Compare}((u,v),(S,S'),(X,Y))$ not true}
\STATE \textbf{return} $\mathrm{CondorcetWinning}(\mathrm{refinedPartition})$ \label{line:partition2}
\ENDIF
\IF{$v \in Z$}
\STATE $L \leftarrow L \cup \{v\}$, go to $(*)$
\ELSE \STATE \textbf{return} $U \cup X$ \label{line:return1}
\ENDIF
\ENDWHILE
\STATE \textbf{return} $U \cup X$ \label{line:return2}
\end{algorithmic}
\caption{CordorcetWinning2}\label{alg:condorcet1}
\end{algorithm}

\begin{lemma}\label{lem:Condorcet2}
After performing $\mathcal{O}(k^2 \cdot \log(k))$ many duels, \emph{CondorcetWinning2} has output a Condorcet winning team.
\end{lemma}

\begin{proof}
We start by showing that the two duels in line \ref{line:firstquery} are feasible. To this end observe that $U,V,X$ and $Y$ are disjoint by construction. To argue about their cardinalities, we consider the two cases of the if condition. First, assume $|T_{\leq i_k}| - k < k - |T_{<i_k}|$. Then \[|U| = |T_{<i_k}| + |T_{i_k}|-2j = |T_{\leq i_k}| - (|T_{\leq i_k}| - k) - j = k-j.\] As $|X|=|Y|=j$, we get that $|U|+|X|=|U|+|Y|=k$. Similarly, for the other case, we have \[|V| = |T_{<i_{2k}}| - |T_{\leq i_k}| + |Z| = |T_{<i_{2k}}| - |T_{\leq i_k}| + 2k = |T_{<i_k}| = k + k - |T_{i_{k}}| = k - j.\] Hence, also $|U|+|X|=|U|+|Y|=k$.

Next, we show that we can find a refined partition in line \ref{line:partition1}. Assume wlog that $V \cup Y \succ U \cup X$ holds and observe that both statements cannot be true as $U \triangleright V$ by construction. Hence, we have $U \cup Y \succ V \cup X$ which implies that we can call $\mathrm{Uncover(Y,X,U,V)}$ which returns a pair $(y,x)$ as well as a witness from $\mathcal{S}_{y,x}$ (or $\mathcal{S}_{x,y}$). Since $x$ and $y$ are from the same indifference class of $\mathscr{T}$, namely $T_{i_k}$, we can call the $\mathrm{NewCut}$ subroutine and obtain a refined partition. 

The call to $\mathrm{Uncover}$ in line \ref{line:warmstart1} is feasible, as the non-satisfaction of the if condition implies that $U \cup X \succ V \cup Y$ and $U \cup Y \succ V \cup X$. 

In line \ref{line:partition2} we can refine the partition $\mathscr{T}$, as, if $\mathrm{Compare}((u,v),(S,S'),(X,Y))$ does not return true, then $\mathrm{Uncover}(X,Y,S\setminus X,S' \setminus Y)$ returns a pair $(x,y)$ with $x \in X$ and $y \in Y$ (or $(y,x)$) together with a witness from $\mathcal{S}_{x,y}$ (or $\mathcal{S}_{y,x}$). Since $x$ and $y$ are both from the same indifference class of $\mathscr{T}$, namely $T_{i_k}$, we can refine $\mathscr{T}$ by calling the $\mathrm{NewCut}$ subroutine. 

Lastly, we show that $U \cup X$ is a Condorcet winning team when the algorithm reaches line \ref{line:return1} or line \ref{line:return2}. We first discuss line 24. First, observe that $U \triangleright V, u \in U$, $v \in V$ and $(S,S')$ is a witness for their relation, that is, $(S,S') \in \mathcal{S}_{uv}$. Moreover, since $\mathrm{Compare}((u,v),(S,S'),(X,Y))$ is true, we know that 
\begin{equation}
v(u) - v(v) > |v(X) - v(Y)|. \label{eq:compareuvXY}
\end{equation} Additionally we know that $v \in V \setminus Z$, which implies that $v \in T_{<i_{2k}}$. Hence, $v$ is in particular contained in the best response against $U \cup X$. Since $Y$ is also guaranteed to be within the best response, we can denote the best response by $V' \cup Y$. Using \cref{eq:compareuvXY} and the fact that $U \triangleright V'$, we get 
\begin{align*}
    v(U \cup X) - v(V' \cup Y) &= v(U \setminus \{u\}) + v(u) + v(X) - v(V' \setminus \{v\}) - v(v) - v(Y) \\ 
    & > v(U \setminus \{u\}) - v(V' \setminus \{v\}) > 0, 
\end{align*}
showing that $U \cup X \succ V' \cup Y$. 

Now, consider the situation in line \ref{line:return2}. This implies that the list $L$ is of length $|T_{\leq i_{2k}}| - 2k +1$ and for each $v \in L$ there exists $u \in U$ such that \begin{equation} v(u) - v(v) > |v(X) - v(Y)|. \label{eq:compareuvXY2}\end{equation} Again, the best response against $U \cup X$ contains $Y$. Denote the best response by $V' \cup Y$. By the size of $L$ we know that $V' \cap L \neq \emptyset$. Let $v$ be a node in the intersection and $u$ be the node for which the algorithm has proven \cref{eq:compareuvXY2}. Due to the same argumentation as before, $U \triangleright V'$ and $v(u') - v(v') > v(X) - v(Y)$ implies $U \cup X \succ V' \cup Y$. 

It remains to argue about the number of duels performed by \emph{CondorcetWinning2}. Again, it is clear that the partition $\mathscr{T}$ can be refined at most $\mathcal{O}(k)$ times. Per refinement, the is one additional call to \emph{Uncover} which is bounded by $\mathcal{O}(\log(k))$ duels. Moreover, the iterations of the while loop can be bounded by $\mathcal{O}(k)$. Within one iteration the algorithm performs \emph{Compare} (requiring $\mathcal{O}(1)$ duels) and \emph{Uncover} (requiring $\mathcal{O}(\log(k))$ duels). Putting everything together the number of duels can hence be bounded by $\mathcal{O}(k^2 \log(k))$. 
\end{proof}

Putting Lemma \ref{lem:Condorcet1} and Lemma \ref{lem:Condorcet2} together clearly yields the proof of Lemma \ref{lemCondorcetWinning}. 

\lemCondorcetWinning*

\medskip
\paragraph{Extension to a Stochastic Environment} In the following we sketch how we can reduce any stochastic instance satisfying $|P_{A,B}-1/2| \in [1/2 + \theta, 1]$ to our deterministic setting. 
 To achieve such a reduction, 
simulate each deterministic duel by  $\mathcal{O}(\frac{\ln{m/\delta}}{\theta^2})$ stochastic duels to determine the duel's winner with probability at least $1-\delta/m$, where  $\mathcal{O}(m)$ is the sample complexity of an algorithm that finds a Condorcet winning team in the deterministic case. An invocation of Chernoff-Hoeffding concentration bound yields that each duel's winner is correctly determined by this simulation with probability at least $1-\delta/m$, and applying union bound over the total number of duels results in an algorithm that requires  $\mathcal{O}(m\frac{\ln{m/\delta}}{\theta^2})$ team duels to identify a Condorcet winning team with probability at least $1-\delta$.

\section{Additive Total Orders} \label{sec:additiveLinear}

In the following we provide a sufficient condition for assigning values to players in a way that complies with a total order on teams, assuming that each team has value of the cumulative values of it's players and that team $A$ is better than team $B$ if and only if the value of $A$ is larger than the value of $B$. Formally:

\textbf{Given:} A set of players $[n]$ and a total order $\succ$ on the subsets of size $k$. 

\textbf{Question:} Do there exist values for the players representing this order? Or more precisely, does the following system of linear inequalities have a feasible solution? 

We denote define $\mathcal{D} = \{(A,B) \mid A,B\text{ are teams, } A \succ B\}$. 

\begin{align*}
\sum_{b \in B} x_b - \sum_{a \in A} x_a &\leq -1 \text{ for all } (A,B) \in \mathcal{D}\\
x_a &\geq 0 \text{ for all } a \in [n]
\end{align*}

We remark that, alternatively to $-1$ on the right hand side, we could have chosen any other negative number. 

The following is a variant of Farkas Lemma:\\ 
\begin{lemma}[Farkas' Lemma \protect\citepAppendix{farkas1902theorie}]
Let $n,m \in \mathbb{N}$, $A \in \mathbb{R}^{n\times m}$ and $b \in \mathbb{R}^m$. Then, exactly one of the following is true. 
\begin{enumerate}
    \item $\exists \; x \in \mathbb{R}^n, Ax \leq b, x\geq 0$
    \item $\exists\; y \in \mathbb{R}^m, y^T A \geq 0, y \geq 0$ and $y^T b <0.$
\end{enumerate}
\end{lemma}

Imagine the system above in matrix form $Ax$, then the system $y^TA\geq 0,y^Tb<0,y\geq 0$ looks as follows: 

\begin{align*}
\sum_{(A,B) \in \mathcal{D} : i \in B} y_{AB} - \sum_{(A,B) \in \mathcal{D} : i \in A} y_{AB} &\geq 0 \text{ for all players } i \in [n] \\ 
y_{AB} &\geq 0  \text{ for all } (A,B) \in \mathcal{D} \\ 
\sum_{(A,B) \in \mathcal{D}} y_{AB} &> 0 
\end{align*}

Assume the second system does have a feasible solution $y \geq 0$. In particular, there exists one pair $A \succ B$ for which $y_{AB} >0$. We can assume wlog that this solution is rational and by scaling it up that it is integer. 

We define the following condition: 

\textbf{Condition (*)}
There exist $\mathcal{A} = \{A_1,\dots, A_m\}$ and $\mathcal{B} = \{B_1,\dots, B_m\}$ satisfying the following two conditions: 
\begin{itemize}
\item[(i)] $A_j \succ B_j$ for all $j \in [m]$
\item[(ii)] Let $n^{\mathcal{A}}_i$ be the number of times that player $i$ is included in some element of $\mathcal{A}$. Define $n^{\mathcal{B}}_i$ analogously. Then, $n^{\mathcal{A}}_i = n^{\mathcal{B}}_i$ for all players $i \in [n]$. 
\end{itemize}

\begin{claim}
The second system of linear inequalities has a feasible solution if and only if $(*)$ is satisfied. 
\end{claim}

\begin{proof}
$``\Rightarrow"$ Assume the second system has a feasible (and wlog integral) solution $y$. We construct $\mathcal{A}$ and $\mathcal{B}$ as follows: For each pair $A \succ B$ for which $y_{AB}>0$, add exactly $y_{AB}$ copies of A and B to $\mathcal{A}$ and $\mathcal{B}$, respectively. The first constraints for condition $(*)$ is clearly satisfied. Now, assume for contradiction that there exists a player $i \in [n] $ for which $n_i^{\mathcal{A}} > n_i^{\mathcal{B}}$ holds. Then, we get \[\sum_{(A,B) \in \mathcal{D} : i \in B} y_{AB} - \sum_{(A,B) \in \mathcal{D}: i \in A} y_{AB} =  n_i^{\mathcal{B}} - n_i^{\mathcal{A}} <0,\] a contradiction to the feasibility of $y$. On the other hand, assume that there exists a player $i \in [n]$ for which $n_i^{\mathcal{A}} < n_i^{\mathcal{B}}$ holds. Observe that \[\sum_{j \in [n]}n_j^{\mathcal{A}} = \sum_{j \in [n]}n_j^{\mathcal{B}} = |\mathcal{A}|k\] and hence \[\sum_{j \in [n] \setminus \{i\}} n_j^{\mathcal{A}} > \sum_{j \in [n] \setminus \{i\}} n_j^{\mathcal{B}},\] which implies that there exists some $i' \in [n] \setminus \{i\}$ with $n_{i'}^{\mathcal{A}} > n_{i'}^{\mathcal{B}}$, a contradiction.

$``\Leftarrow"$ Assume that there exist $\mathcal{A}$ and $\mathcal{B}$ satisfying condition $(*)$. Then, set $y_{A_j,B_j}=  |\{q \in [m] : (A_q,B_q) = (A_j,B_j)\}| $ for all $j \in [m]$ and $y_{A,B}=0$ for all other duels. This is a feasible solution to the second system of inequalities. 
\end{proof}

This directly yields the sufficient condition for a total order to be representable by values. 
\begin{corollary}
There exists a solution to the first system of inequalities if and only if condition $(*)$ does not hold. 
\end{corollary}

\bibliographystyleAppendix{apa-good}
\bibliographyAppendix{refs}
\end{document}